\documentclass[journal]{IEEEtran}

\usepackage[utf8]{inputenc}
\usepackage[T1]{fontenc}
\usepackage{amsmath,amssymb,amsthm}
\usepackage{graphicx}
\usepackage{booktabs}
\usepackage[hyphens]{url}
\usepackage[breaklinks=true,hidelinks,hypertexnames=false]{hyperref}
\usepackage{cite}
\usepackage{xcolor}
\usepackage{multirow}
\usepackage{enumitem}
\usepackage{algorithm}
\usepackage{algorithmic}
\usepackage[section]{placeins}
\usepackage{stfloats}
\usepackage{microtype}

\newtheorem{proposition}{Proposition}
\newtheorem{theorem}{Theorem}
\newtheorem{corollary}{Corollary}

\renewenvironment{proof}[1][Proof]{\noindent\textit{#1.}\space}{\hfill$\blacksquare$\par\medskip}

\begin{document}

\title{PCA-Driven Adaptive Sensor Triage for Edge AI Inference}

\author{
\IEEEauthorblockN{Ankit Hemant Lade\textsuperscript{1}, Sai Krishna Jasti\textsuperscript{2}, Nikhil Sinha\textsuperscript{3}, Indar Kumar\textsuperscript{4}, Akanksha Tiwari\textsuperscript{5}}
\IEEEauthorblockA{Independent Researchers}
\IEEEauthorblockA{\textsuperscript{1}ankitlade12@gmail.com, \textsuperscript{2}jsaikrishna379@gmail.com, \textsuperscript{3}sinha.nikhil77@gmail.com, \textsuperscript{4}indarkarhana@gmail.com, \textsuperscript{5}akankshat2803@gmail.com}
}

\maketitle

% ============================================================
\begin{abstract}
Multi-channel sensor networks in industrial IoT often exceed available bandwidth. We propose \textbf{PCA-Triage}, a streaming algorithm that converts incremental PCA loadings into proportional per-channel sampling rates under a bandwidth budget. PCA-Triage runs in $O(wdk)$ time with zero trainable parameters (0.67\,ms per decision).

We evaluate on 7 benchmarks (8--82 channels) against 9 baselines. PCA-Triage is the best unsupervised method on 3 of 6 datasets at 50\% bandwidth, winning 5 of 6 against every baseline with large effect sizes ($r = 0.71$--$0.91$). On TEP, it achieves F1\,=\,$0.961 \pm 0.001$---within 0.1\% of full-data performance---while maintaining F1\,$>$\,0.90 at 30\% budget. Targeted extensions push F1 to 0.970. The algorithm is robust to packet loss and sensor noise (3.7--4.8\% degradation under combined worst-case).
\end{abstract}

\begin{IEEEkeywords}
sensor triage, bandwidth allocation, incremental PCA, edge AI, IoT, fault detection, adaptive sampling, streaming algorithms
\end{IEEEkeywords}

% ============================================================
\section{Introduction}
\label{sec:intro}

Consider a chemical plant instrumented with 200 sensors monitoring temperatures, pressures, flow rates, and valve positions. Each sensor streams data at 1\,Hz, generating 1.2\,MB per minute. An edge gateway must relay this data over a constrained industrial network to a fault detection system, but the available bandwidth supports only 50\% of the full data volume. The operator faces a deceptively simple question: \textit{which sensors should receive more of the limited bandwidth, and which can tolerate lower sampling rates without compromising fault detection?}

The naive answer---sample every sensor at a uniformly reduced rate---wastes bandwidth on channels that carry redundant or uninformative signals. The ideal allocation is \textit{adaptive}: it shifts bandwidth toward informative channels as operating conditions change.

\subsection{Problem Formulation}

We formalize this as the \textit{sensor triage} problem. Given $d$ sensor channels and a total bandwidth budget $B \in (0, 1]$, allocate a per-channel sampling rate $r_j \in [r_{\min}, 1]$ for each channel $j = 1, \ldots, d$, subject to the budget constraint (Eq.~\ref{eq:budget}):
\begin{equation}
\frac{1}{d} \sum_{j=1}^{d} r_j \leq B
\label{eq:budget}
\end{equation}
such that a downstream fault detection model trained on triaged data maintains accuracy close to one trained on full-rate data. The allocation must be updated online as new data arrives.

This problem sits at the intersection of three well-studied areas---yet none addresses it directly. \textit{PCA-based process monitoring}~\cite{chen2013pca, processmonitoring2024survey, downs1993tep} uses principal component loadings for post-hoc fault diagnosis only. \textit{Streaming PCA algorithms}~\cite{weng2003ccipca, balzano2010grouse, yang2018historypca} output basis matrices rather than channel-level decisions. \textit{Adaptive sampling in IoT}~\cite{benaboud2021adaptive, giordano2023energy} applies the same rate to every channel.

\subsection{Key Insight: Correlation Structure}

The core insight behind PCA-Triage is that PCA naturally captures inter-channel correlations that variance-based methods miss. Fig.~\ref{fig:correlation} shows the TEP sensor correlation matrix during fault-free operation. Several sensor clusters exhibit $|r| > 0.7$ correlation (e.g., temperature sensors xmeas\_9--xmeas\_11, flow sensors xmeas\_1--xmeas\_4). When channels A and B are highly correlated, PCA detects the redundancy and concentrates importance on the more informative channel, freeing bandwidth for independent sensors. Variance-based allocation, by contrast, treats each channel independently and cannot exploit this structure.

\begin{figure}[!htbp]
\centering
\includegraphics[width=\columnwidth]{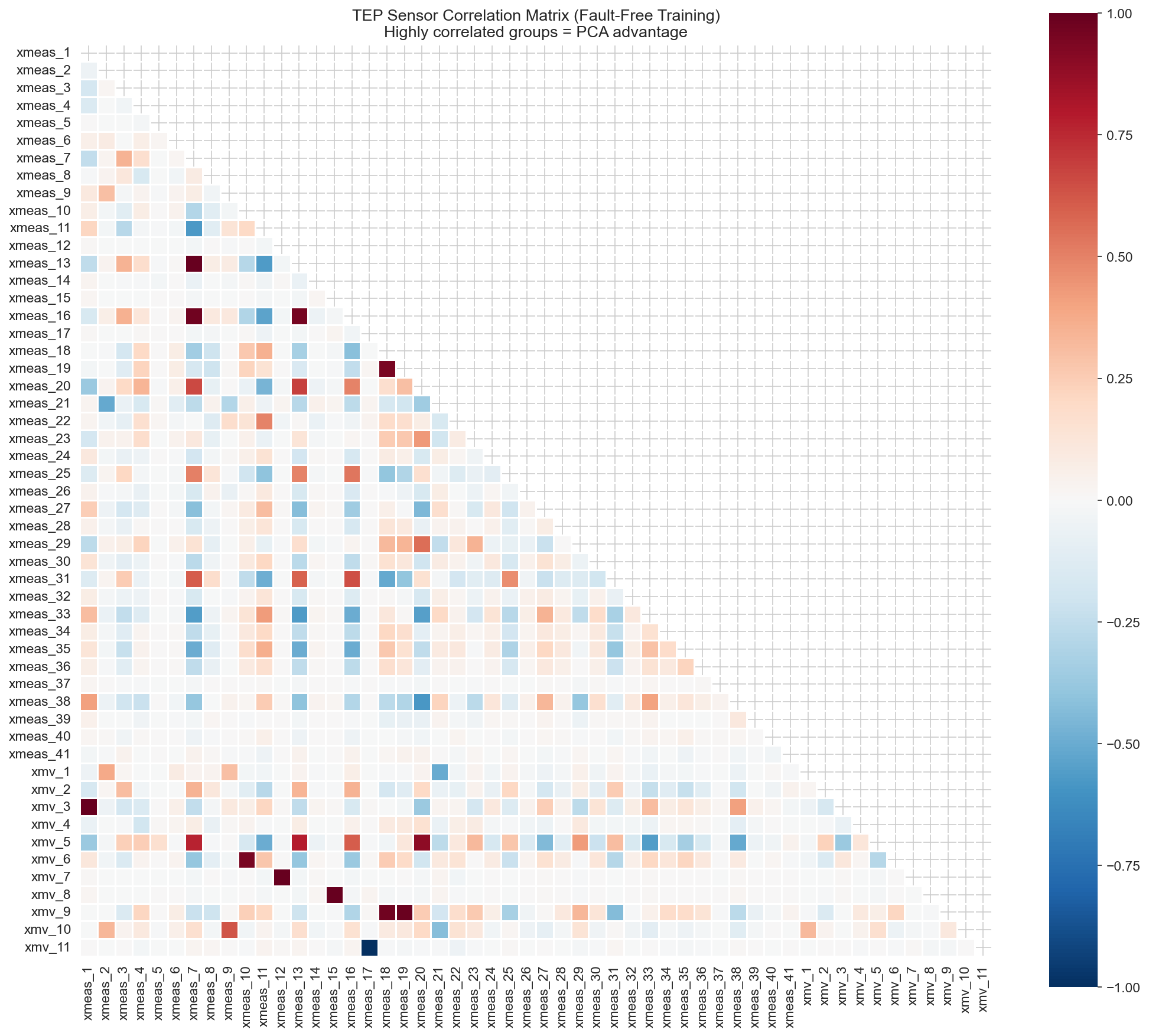}
\caption{TEP sensor correlation matrix during fault-free operation. Highly correlated clusters (dark red/blue blocks) represent redundancy that PCA-Triage exploits for efficient bandwidth allocation.}
\label{fig:correlation}
\end{figure}

\subsection{Contributions}

\begin{enumerate}[leftmargin=*]
\item \textbf{Algorithm.} We introduce PCA-Triage, a streaming algorithm that converts incremental PCA loadings into proportional per-channel sampling rates under a bandwidth budget. The algorithm is unsupervised, runs at $O(wdk)$ per window, and requires zero trainable parameters. We further present three extensions---hybrid PCA+variance importance scoring, linear interpolation reconstruction, and power-law sharpened rate allocation---and quantify each contribution via ablation.

\item \textbf{Theoretical analysis.} We provide formal guarantees on budget constraint satisfaction (Proposition~\ref{prop:budget}), importance score convergence (Proposition~\ref{prop:convergence}), and PCA's advantage over variance-based allocation in the presence of inter-channel correlation (Theorem~\ref{thm:correlation}).

\item \textbf{Comprehensive empirical validation.} We evaluate on 7 benchmarks (8--82 sensors) against 9 baselines across 1,000+ experiments with 3--5 seeds, 3 classifiers, and 17 distinct analyses, with full statistical testing (Wilcoxon signed-rank, Friedman ranking). PCA-Triage is the best unsupervised method on 3 of 6 Pareto-evaluated datasets at 50\% bandwidth, and competitive on the remainder.

\item \textbf{Edge viability.} PCA-Triage completes each decision in 0.67\,ms on a single CPU core, adapts to fault onset with speed controlled by $\lambda$ (0--3 windows at $\lambda \leq 0.80$), and scales to 500 channels under the 5\,ms edge target.
\end{enumerate}

% ============================================================
\section{Related Work}
\label{sec:related}

We organize the literature along two axes---\textit{static vs.\ streaming} and \textit{fixed-rule vs.\ data-driven}---to identify the gap our method fills (Table~\ref{tab:positioning}).

\subsection{PCA for Process Monitoring}

PCA has been the workhorse of industrial process monitoring for over two decades~\cite{chen2013pca, processmonitoring2024survey}. The standard approach monitors Hotelling's $T^2$ and Squared Prediction Error (SPE) statistics. On the Tennessee Eastman Process~\cite{downs1993tep, rieth2017tep}, PCA-based methods achieve fault detection rates exceeding 93\%~\cite{chen2013pca}. Loading vectors reveal which sensors contribute most to each principal component---used for \textit{post-hoc} fault diagnosis~\cite{mnassri2015feature}.

However, all existing PCA-based methods assume full-resolution data is available. None address bandwidth-constrained scenarios where the system must decide \textit{which channels to prioritize before data collection}.

\subsection{Streaming PCA and Online Feature Selection}

Oja's rule~\cite{oja1982simplified} provides the foundational online PCA update. CCIPCA~\cite{weng2003ccipca} extends to multiple components at $O(dk)$ cost. GROUSE~\cite{balzano2010grouse} tracks subspaces on the Grassmannian manifold. History PCA~\cite{yang2018historypca} achieves faster convergence. Balzano et al.~\cite{balzano2018streaming} survey these algorithms comprehensively. Recent work on low-precision streaming PCA~\cite{lowprecision2025streaming} demonstrates continued algorithmic progress. The IncrementalPCA implementation in scikit-learn~\cite{ross2008incremental, levy2000sequential} uses a batch-wise SVD update that is numerically stable and widely deployed. Randomized SVD~\cite{halko2011finding} offers further speedups for large $d$.

Online feature selection methods---Alpha-investing~\cite{zhou2006alpha}, SAOLA~\cite{yu2014saola}, OSFS~\cite{wu2013osfs}, OSFSW~\cite{you2018osfsw}---produce \textit{binary} include/exclude decisions and require supervised labels~\cite{zaman2022streaming}. Even the most recent OSSFS~\cite{zhou2025ossfs} (2025) remains binary and supervised.

None of these methods produce \textit{proportional} per-channel sampling rates under a budget constraint.

\subsection{Adaptive Sampling in IoT}

Ben-Aboud et al.~\cite{benaboud2021adaptive} use Kalman predictions to adapt temporal sampling intervals. Giordano et al.~\cite{giordano2023energy} achieve 85--95\% of optimal throughput with 527 CPU cycles per day. ML-DSRA~\cite{mldsra2024parametric} auto-tunes temporal sampling parameters. Compressive sensing~\cite{cs2017survey} exploits signal sparsity but requires expensive reconstruction.

All methods treat sensors identically---the same temporal rate for every channel. Our method answers \textit{which channels} deserve more bandwidth.

\subsection{Edge AI and Attention-Based Methods}

Edge AI research~\cite{gill2024edgeai} focuses on model-side optimization (compression, distillation, split computing). Attention mechanisms~\cite{cao2021mafs, vtt2024transformer, xu2023dcff} learn channel importance but at $O(d^2)$ cost with hundreds of thousands of parameters---too expensive for edge deployment.

Our method optimizes the \textit{data side}---complementary to model compression. PCA-Triage runs at $O(wdk)$ with zero trainable parameters.

\subsection{Positioning}

Table~\ref{tab:positioning} summarizes the positioning of PCA-Triage relative to existing approaches across two axes: static vs.\ streaming and fixed-rule vs.\ data-driven. Recent work has explored correlation-aware sensor selection~\cite{bacciu2016unsupervised, ghosh2021adaptive, ghosh2021edge} and adaptive sampling-rate allocation under resource budgets~\cite{yang2023freqsense}. PCA-Triage builds on this line of work by using incremental PCA specifically as the importance engine, producing proportional (not binary) per-channel rates under a total bandwidth constraint, in a zero-parameter streaming setting.

\begin{table}[!htbp]
\centering
\caption{Positioning matrix. PCA-Triage combines streaming operation with data-driven proportional allocation (bottom-right cell). Prior work~\cite{bacciu2016unsupervised, ghosh2021adaptive} explores correlation-aware sensor selection but in batch or binary settings.}
\begin{tabular}{@{}lcc@{}}
\toprule
& \textbf{Static/Batch} & \textbf{Streaming/Adaptive} \\
\midrule
\textbf{Fixed Rules} & Uniform, SoD & DSRA, AIMD \\
\textbf{Data-Driven} & Batch PCA, Offline FS & \textbf{PCA-Triage (Ours)} \\
\bottomrule
\end{tabular}
\label{tab:positioning}
\end{table}

% ============================================================
\section{Method}
\label{sec:method}

\subsection{System Architecture}

Fig.~\ref{fig:architecture} shows the PCA-Triage pipeline. Raw sensor data flows through a sliding window buffer into IncrementalPCA, which produces loadings and singular values. The importance scorer converts these into per-channel scores, and the rate allocator distributes the bandwidth budget proportionally. The full procedure is summarized in Algorithm~\ref{alg:pcatriage}.

\begin{figure*}[!htbp]
\centering
\includegraphics[width=0.85\textwidth]{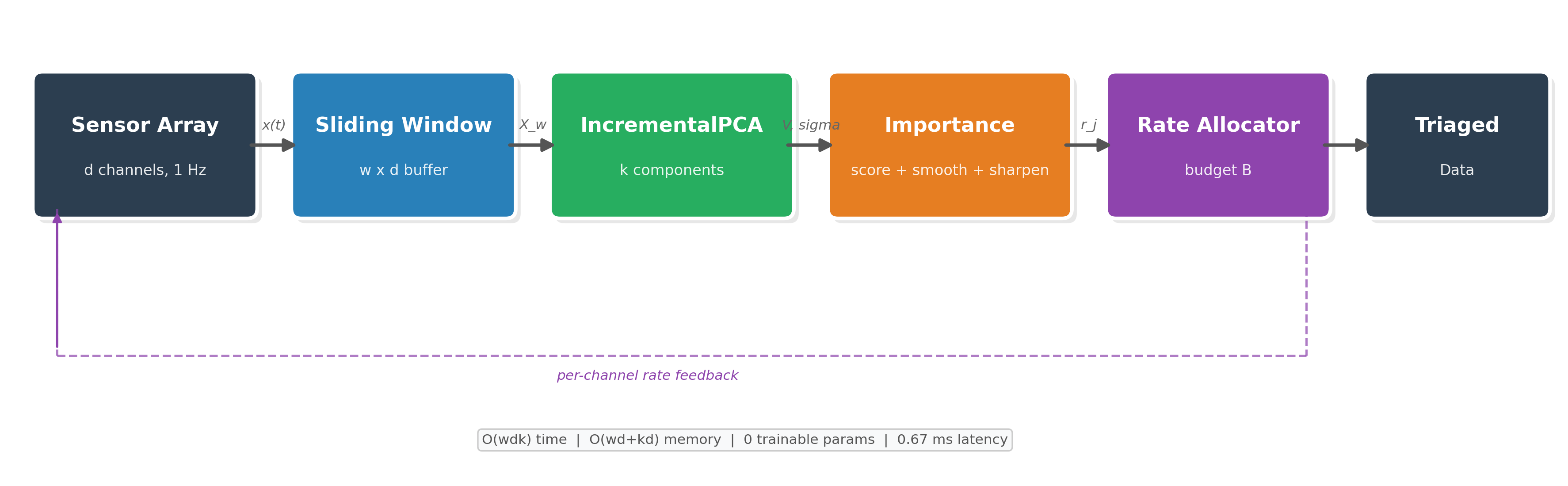}
\caption{PCA-Triage system architecture. Sensor data passes through a sliding window into IncrementalPCA, which extracts loadings $\mathbf{V}$ and singular values $\boldsymbol{\sigma}$. Importance scores $s_j$ are computed from weighted loadings and fed to the rate allocator, which distributes the bandwidth budget $B$ across channels. Time: $O(wdk)$; Memory: $O(wd + kd)$; Trainable parameters: 0.}
\label{fig:architecture}
\end{figure*}

\subsection{Channel Importance Scoring}

For a window $\mathbf{X}_w \in \mathbb{R}^{w \times d}$, PCA-Triage fits IncrementalPCA to extract loadings $\mathbf{V} \in \mathbb{R}^{k \times d}$ and singular values $\boldsymbol{\sigma} \in \mathbb{R}^k$. The importance score for channel $j$:
\begin{equation}
s_j = \sum_{i=1}^{k} \sigma_i \cdot V_{ij}^2
\label{eq:importance}
\end{equation}
The weighting by $\sigma_i$ ensures that channels contributing to high-variance components receive proportionally higher importance. Squaring the loading $V_{ij}^2$ ensures non-negativity and penalizes small contributions.

\textbf{Hybrid scoring.} On datasets with limited inter-channel correlation (e.g., SKAB with 8 sensors), pure PCA scores may not capture enough structure. We optionally blend PCA importance with per-channel variance:
\begin{equation}
s_j^{\text{hybrid}} = \alpha \cdot s_j^{\text{PCA}} + (1 - \alpha) \cdot \frac{\text{Var}(x_j)}{\sum_{j'} \text{Var}(x_{j'})}
\label{eq:hybrid}
\end{equation}
where $\alpha \in [0, 1]$ controls the blend. High $\alpha$ (e.g., 0.8) is used for richly correlated datasets (TEP), while low $\alpha$ (e.g., 0.0--0.4) is used for low-correlation or few-channel datasets (SKAB, PSM). Per-dataset $\alpha$ values are: TEP $\alpha{=}0.8$, SMD $0.75$, MSL $0.7$, PSM $0.4$, HAI $0.05$, SKAB $0.15$, SWaT $0.7$.

Scores are smoothed via exponential moving average with forgetting factor $\lambda \in (0, 1]$:
\begin{equation}
\bar{s}_j^{(t)} = \lambda \cdot \bar{s}_j^{(t-1)} + (1 - \lambda) \cdot s_j^{(t)}
\label{eq:smoothing}
\end{equation}
The parameter $\lambda$ controls the accuracy--adaptivity trade-off at the \textit{scoring} level. Note that the PCA model itself always updates via \texttt{partial\_fit}, so importance scores still evolve; $\lambda$ controls how aggressively the smoothed scores track these changes. $\lambda = 1.0$ retains the cumulative score average (most stable), while $\lambda \to 0$ uses only the current window's scores (most responsive but noisy).

\subsection{Rate Allocation}

Given smoothed scores $\bar{\mathbf{s}}$ and budget $B$, we first apply \textit{power-law sharpening} to concentrate bandwidth on the most important channels:
\begin{equation}
\tilde{s}_j = \frac{\bar{s}_j^{\,\gamma}}{\sum_{j'} \bar{s}_{j'}^{\,\gamma}}
\label{eq:sharpening}
\end{equation}
where $\gamma \geq 1$ is the sharpness exponent. When $\gamma = 1$, allocation is proportional to importance (standard). As $\gamma \to \infty$, allocation concentrates on the highest-importance channels (winner-take-all). We find $\gamma \in [1.5, 3.0]$ works best, depending on dataset channel count and correlation structure. Rates are then allocated as:
\begin{equation}
r_j = r_{\min} + \tilde{s}_j \cdot (B \cdot d - r_{\min} \cdot d)
\label{eq:allocation}
\end{equation}
with clipping to $[r_{\min}, 1]$. The minimum rate $r_{\min}$ prevents any channel from being completely silenced, ensuring the downstream model always receives some signal from every sensor.

\subsection{Data Acquisition and Reconstruction}

At each time step $t$, sample $x_t[j]$ is retained with probability $r_j$. Missing values are reconstructed via linear interpolation between the nearest observed values on each channel. This preserves signal continuity better than forward-fill (zero-order hold), improving downstream F1 by 0.4--1.8\% across datasets (Table~\ref{tab:recon}).

\subsection{Algorithm}

\begin{algorithm}[t]
\caption{PCA-Triage}
\label{alg:pcatriage}
\begin{algorithmic}[1]
\REQUIRE Stream of windows $\mathbf{X}_w$, budget $B$, components $k$, forgetting $\lambda$, minimum rate $r_{\min}$, blend $\alpha$, sharpness $\gamma$
\STATE Initialize IncrementalPCA($k$), $\bar{\mathbf{s}} \gets$ None
\FOR{each window $\mathbf{X}_w$}
  \STATE \texttt{partial\_fit}($\mathbf{X}_w$) $\to \mathbf{V}, \boldsymbol{\sigma}$
  \STATE $s_j \gets \alpha \sum_i \sigma_i V_{ij}^2 + (1{-}\alpha)\,\text{Var}(x_j)$ \hfill $\triangleright$ Hybrid (Eq.~\ref{eq:hybrid})
  \STATE $\bar{\mathbf{s}} \gets \lambda \bar{\mathbf{s}} + (1-\lambda) \cdot \text{normalize}(\mathbf{s})$ \hfill $\triangleright$ Smooth (Eq.~\ref{eq:smoothing})
  \STATE $\tilde{s}_j \gets \bar{s}_j^{\,\gamma} / \sum \bar{s}^{\,\gamma}$ \hfill $\triangleright$ Sharpen (Eq.~\ref{eq:sharpening})
  \STATE $r_j \gets r_{\min} + \tilde{s}_j \cdot (Bd - r_{\min} d)$ \hfill $\triangleright$ Allocate (Eq.~\ref{eq:allocation})
  \STATE Clip $r_j$ to $[r_{\min}, 1]$
  \STATE Keep sample $x_t[j]$ with probability $r_j$ \hfill $\triangleright$ Acquire
  \STATE Linear-interpolate missing values \hfill $\triangleright$ Reconstruct
\ENDFOR
\end{algorithmic}
\end{algorithm}

\subsection{Complexity Analysis}

\textbf{Time:} $O(wdk)$ per window, dominated by the IncrementalPCA \texttt{partial\_fit} operation which performs a truncated SVD on a $(w+k) \times d$ matrix. The importance scoring and rate allocation steps are $O(dk)$ and $O(d)$ respectively.

\textbf{Memory:} $O(wd + kd)$---the window buffer $\mathbf{X}_w$ and PCA components. No historical windows are stored.

\textbf{Parameters:} 0 trainable. All quantities ($\mathbf{V}, \boldsymbol{\sigma}, \bar{\mathbf{s}}, \mathbf{r}$) are computed from data.

For TEP ($d{=}52$, $k{=}10$, $w{=}50$): ${\sim}26{,}000$ FLOPs per window---well within the budget of any microcontroller.

% ============================================================
\section{Theoretical Analysis}
\label{sec:theory}

We provide formal guarantees for four properties of PCA-Triage: budget feasibility, importance score convergence, the distinction between correlated and independent channels under PCA-based scoring, and adaptation rate after regime changes.

\begin{proposition}[Budget Feasibility]
\label{prop:budget}
For any non-negative smoothed importance scores $\bar{s}_j \geq 0$ and budget $B \in (r_{\min}, 1]$, the rates produced by Eq.~\ref{eq:allocation} (before clipping) satisfy $\frac{1}{d}\sum_{j=1}^d r_j = B$.
\end{proposition}

\begin{proof}
Summing over all channels:
\begin{align}
\sum_{j=1}^d r_j &= \sum_{j=1}^d \left[ r_{\min} + \frac{\bar{s}_j}{\sum_{j'} \bar{s}_{j'}} \cdot (Bd - r_{\min} d) \right] \nonumber \\
&= r_{\min} d + (Bd - r_{\min} d) \cdot \underbrace{\sum_j \frac{\bar{s}_j}{\sum_{j'} \bar{s}_{j'}}}_{=\,1} \nonumber \\
&= r_{\min} d + Bd - r_{\min} d = Bd
\end{align}
Therefore $\frac{1}{d}\sum_j r_j = B$. Clipping to $[r_{\min}, 1]$ can only reduce rates, so the post-clipping budget satisfies $\frac{1}{d}\sum_j r_j \leq B$.
\end{proof}

\begin{proposition}[Importance Score Convergence]
\label{prop:convergence}
If IncrementalPCA converges to the true top-$k$ eigenspace (i.e., $\mathbf{V}^{(t)} \to \mathbf{V}^*$ and $\boldsymbol{\sigma}^{(t)} \to \boldsymbol{\sigma}^*$ as $t \to \infty$) under a stationary distribution, then the smoothed importance scores converge: $\bar{s}_j^{(t)} \to s_j^* = \sum_i \sigma_i^* (V_{ij}^*)^2$ for all $j$.
\end{proposition}

\begin{proof}
Under stationarity, the raw scores converge: $s_j^{(t)} \to s_j^*$. The exponential moving average with $\lambda \in (0, 1)$ satisfies the standard EMA error bound:
\begin{equation}
|\bar{s}_j^{(t)} - s_j^*| \leq \lambda^t |\bar{s}_j^{(0)} - s_j^*| + \sup_{\tau \leq t} |s_j^{(\tau)} - s_j^*|
\end{equation}
The first term vanishes geometrically in $\lambda^t$ (initialization bias). The second term vanishes because $s_j^{(t)} \to s_j^*$ under stationarity. Therefore $\bar{s}_j^{(t)} \to s_j^*$.
\end{proof}

\begin{theorem}[PCA Distinguishes Correlated from Independent Channels]
\label{thm:correlation}
Let channels $a$ and $b$ have equal marginal variance $\text{Var}(X_a) = \text{Var}(X_b) = \sigma^2$ but correlation $\rho_{ab} \neq 0$. Let channel $c$ be independent with $\text{Var}(X_c) = \sigma^2$. Then:
\begin{enumerate}[leftmargin=*]
\item Variance-based allocation assigns $r_a = r_b = r_c$ (equal rates), unable to distinguish correlated from independent channels.
\item With $k = d$ components, PCA-based allocation also assigns equal rates. However, with $k < d$ (the standard operating regime), PCA concentrates importance on channels participating in the top-$k$ eigenspace, assigning $s_c = 0$ to channels whose variance lies entirely in dropped components.
\end{enumerate}
\end{theorem}

\begin{proof}
Under equal variance, a variance-based allocator computes identical scores for all three channels, yielding uniform allocation---it cannot distinguish correlated from independent channels.

For PCA, the covariance matrix of $(X_a, X_b, X_c)$ is:
\begin{equation}
\mathbf{C} = \sigma^2 \begin{pmatrix} 1 & \rho & 0 \\ \rho & 1 & 0 \\ 0 & 0 & 1 \end{pmatrix}
\end{equation}
The eigenvalues are $\sigma^2(1+\rho)$, $\sigma^2(1-\rho)$, $\sigma^2$. The first eigenvector is $(1/\sqrt{2}, 1/\sqrt{2}, 0)$, the second is $(1/\sqrt{2}, -1/\sqrt{2}, 0)$, and the third is $(0, 0, 1)$.

Computing importance scores (Eq.~\ref{eq:importance}) with $k=3$ (all components):
\begin{align}
s_a = s_b &= \tfrac{1}{2}\sigma^2(1+\rho) + \tfrac{1}{2}\sigma^2(1-\rho) = \sigma^2 \\
s_c &= \sigma^2
\end{align}
With all three components, scores are equal---no advantage over variance. However, with $k=2$ (retaining ${\sim}95\%$ variance when $\rho$ is large), the third component is dropped:
\begin{align}
s_a = s_b &= \tfrac{1}{2}\sigma^2(1+\rho) + \tfrac{1}{2}\sigma^2(1-\rho) = \sigma^2 \\
s_c &= 0
\end{align}
The independent channel $c$ loses all importance because its variance is captured entirely by the dropped component. Budget is redirected to the correlated pair $a, b$.
\end{proof}

\textbf{Interpretation.} Theorem~\ref{thm:correlation} reveals both the power and the risk of PCA-based allocation. The power: PCA detects correlation structure that variance-based methods cannot, enabling non-uniform allocation among equal-variance channels. The risk: with $k < d$, channels whose variance lies outside the top-$k$ eigenspace receive zero importance and only the minimum rate $r_{\min}$---even if they carry unique, fault-relevant information. This motivates two design choices: (1)~moderate $k$ to avoid starving independent channels (Sec.~\ref{sec:ablation} shows $k \in [3, 10]$ is robust), and (2)~the hybrid scorer (Eq.~\ref{eq:hybrid}), which blends PCA importance with per-channel variance to preserve a baseline importance for channels outside the dominant eigenspace.

\begin{corollary}[Reconstruction Error Bound]
\label{cor:recon}
Under forward-fill reconstruction, the per-channel expected squared error for channel $j$ is bounded by:
\begin{equation}
\mathbb{E}[(x_t[j] - \hat{x}_t[j])^2] \leq (1 - r_j) \cdot \Delta_j^2
\end{equation}
where $\Delta_j^2 = \mathbb{E}[(x_t[j] - x_{t-1}[j])^2]$ is the channel's step variance. PCA-Triage minimizes total reconstruction error by assigning higher rates to channels with both high PCA importance \textit{and} high step variance.
\end{corollary}

\textbf{Regret intuition.} Let $\epsilon_k = 1 - \sum_{i=1}^k \lambda_i / \sum_{i=1}^d \lambda_i$ denote the fraction of variance not captured by the top-$k$ components. As $k$ increases, PCA-Triage's importance scores better approximate the true channel structure, and the allocation approaches oracle quality. For TEP with $k{=}10$, the top-10 components explain ${\sim}95\%$ of variance ($\epsilon_k \approx 0.05$), suggesting that the allocation is near-optimal. This aligns with the empirical finding that PCA-Triage at 50\% budget achieves F1 within 0.1\% of full-data performance.

\begin{proposition}[Adaptation Rate]
\label{prop:adaptation}
After an abrupt regime change at time $t_0$ where the true importance shifts from $\mathbf{s}^{\text{old}}$ to $\mathbf{s}^{\text{new}}$, the smoothed importance scores satisfy:
\begin{equation}
\|\bar{\mathbf{s}}^{(t_0 + \tau)} - \mathbf{s}^{\text{new}}\|_2 \leq \lambda^\tau \|\mathbf{s}^{\text{old}} - \mathbf{s}^{\text{new}}\|_2 + \delta_{\text{PCA}}(\tau)
\end{equation}
where $\delta_{\text{PCA}}(\tau)$ is the PCA model's convergence error after $\tau$ windows of the new regime. For $\lambda = 0.85$, the smoothing bias halves every ${\sim}4.3$ windows ($\tau_{1/2} = \log(0.5) / \log(\lambda)$).
\end{proposition}

\begin{proof}
After the regime change, the EMA update gives $\bar{s}_j^{(t_0+\tau)} = \lambda^\tau \bar{s}_j^{(t_0)} + (1-\lambda) \sum_{l=0}^{\tau-1} \lambda^l s_j^{(t_0+\tau-l)}$. Assuming the PCA model converges to the new regime within $\delta_{\text{PCA}}(\tau)$ error, the raw scores satisfy $\|s_j^{(t)} - s_j^{\text{new}}\| \leq \delta_{\text{PCA}}(\tau)$ for $t > t_0$. Combining with the geometric decay of the initialization bias $\lambda^\tau \|\mathbf{s}^{\text{old}} - \mathbf{s}^{\text{new}}\|_2$ yields the bound.
\end{proof}

\textbf{Interpretation.} The adaptation rate is controlled by two factors: (1) the forgetting factor $\lambda$, which governs how quickly the EMA forgets old scores, and (2) the PCA model's re-convergence speed $\delta_{\text{PCA}}$. With $\lambda = 0.80$, the EMA bias halves every ${\sim}3$ windows ($\tau_{1/2} = \log 0.5 / \log \lambda$), enabling 0--3 window reaction (Fig.~\ref{fig:reaction_lambda}). With $\lambda = 0.85$, halving takes ${\sim}4$ windows but empirical reaction time is longer (${\sim}19$ windows at the 20\% shift threshold; Table~\ref{tab:reaction}). With $\lambda = 1.0$, the EMA never forgets, requiring the PCA model alone to adapt---slower but most stable.

% ============================================================
\section{Experimental Setup}
\label{sec:setup}

\subsection{Datasets}

We evaluate on 7 benchmark datasets spanning 7 application domains (Table~\ref{tab:datasets}), including 6 real-world benchmarks and 1 synthetic dataset calibrated to match the SWaT testbed properties. Fig.~\ref{fig:tep_overview} shows representative TEP sensor traces during fault-free operation, illustrating the diversity of signal characteristics across channels.

\begin{table}[t]
\centering
\caption{Datasets used in evaluation. $\dagger$ = synthetic stand-in calibrated to match published dataset properties (real data requires institutional access agreement).}
\begin{tabular}{@{}llrrl@{}}
\toprule
\textbf{Dataset} & \textbf{Domain} & \textbf{Ch.} & \textbf{Samples} & \textbf{Source} \\
\midrule
TEP & Chemical process & 52 & 250K & \cite{downs1993tep} \\
SMD & Server machines & 38 & 388K & \cite{su2019omnianomaly} \\
MSL & Spacecraft telemetry & 55 & 132K & \cite{hundman2018detecting} \\
PSM & Server metrics & 25 & 220K & \cite{abdulaal2021ransyn} \\
HAI & Industrial control & 82 & 259K & \cite{shin2020hai} \\
SKAB & Water circulation & 8 & 47K & Skoltech \\
SWaT$^\dagger$ & Water treatment & 51 & 500K & iTrust \\
\bottomrule
\end{tabular}
\label{tab:datasets}
\end{table}

\begin{figure}[!htbp]
\centering
\includegraphics[width=\columnwidth]{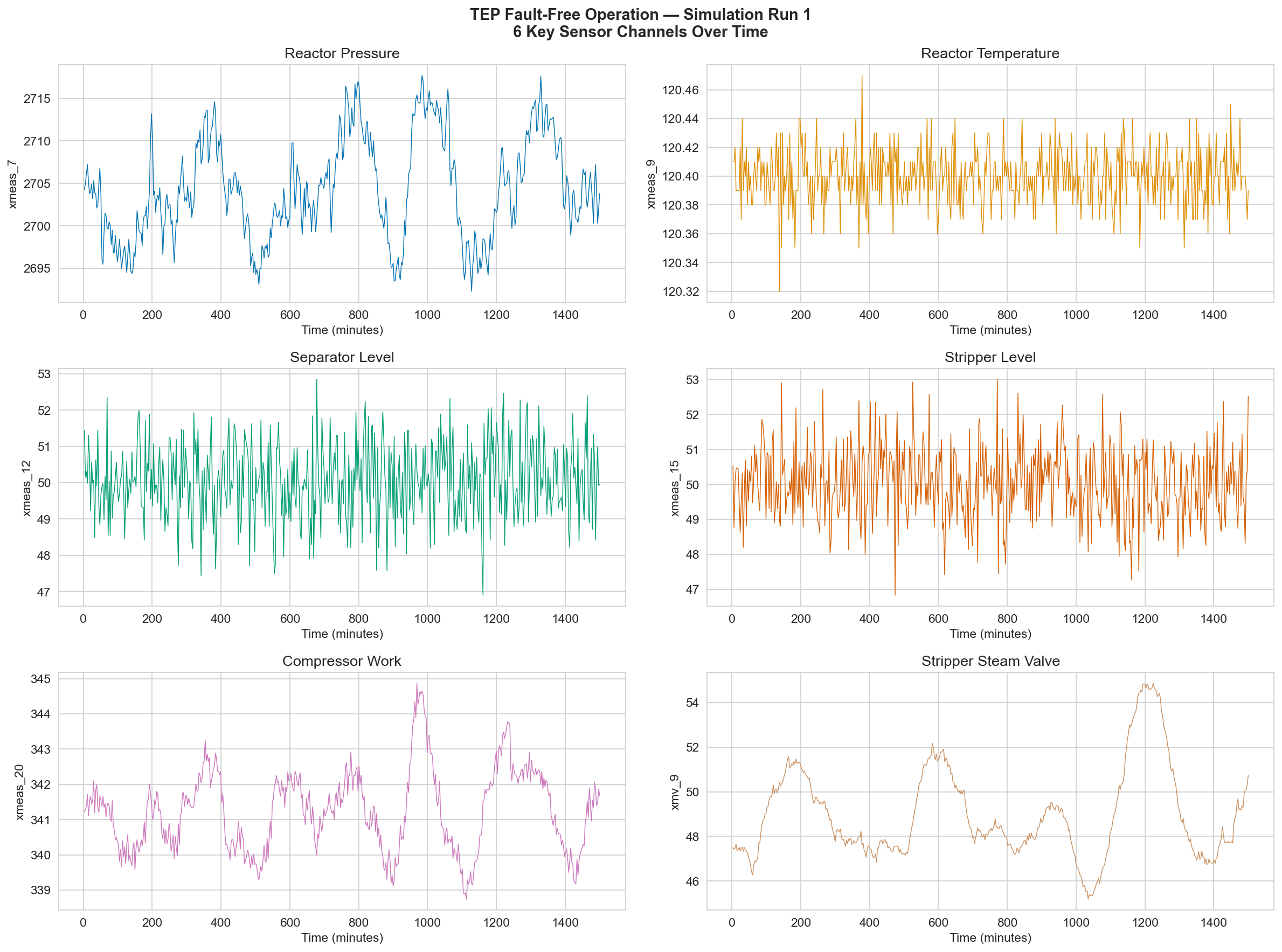}
\caption{TEP fault-free operation: 6 representative sensor channels showing diverse signal characteristics---smooth trends (Reactor Pressure), noisy oscillations (Separator Level), and step-like control signals (Stripper Steam Valve).}
\label{fig:tep_overview}
\end{figure}

\subsection{Baselines}

We compare against nine methods spanning the spectrum from trivial to learned to supervised:

\begin{enumerate}[leftmargin=*]
\item \textbf{Uniform}: Same sampling rate $r_j = B$ for all channels.
\item \textbf{Threshold}: Binary active/inactive based on rolling variance exceeding a percentile threshold.
\item \textbf{Variance}: Proportional to rolling variance $\text{Var}_w(X_j)$.
\item \textbf{Random Dropout}: Randomly drop channels with probability $1-B$.
\item \textbf{Autoencoder}: A single-hidden-layer autoencoder ($d \to 10 \to d$, trained per window with 50 SGD epochs) uses per-channel reconstruction error as importance.
\item \textbf{Mutual Information (supervised)}: Proportional to mutual information with fault labels. Requires labeled data---included as an upper bound.
\item \textbf{LSTM-Attention}: An LSTM with channel attention (14K params on TEP), trained per-window via reconstruction loss. Attention weights serve as channel importance scores.
\item \textbf{Transformer-Attention}: A single Transformer encoder layer where channels are tokens (12K params). Mean self-attention received per channel serves as importance.
\item \textbf{OGD (Online Gradient Descent)}: A regret-optimal online allocation baseline that updates rates via gradient steps on per-channel reconstruction error.
\end{enumerate}

Methods 1--6 are used in the main Pareto comparison (Experiments~1--3). Methods 7--9 are evaluated in dedicated experiments (Experiments~14, 17). The main comparison uses Random Forest~\cite{breiman2001random} with $n{=}100$ trees as the downstream classifier for all methods; dedicated experiments may use $n{=}200$ (noted where applicable). All implementations use scikit-learn~\cite{pedregosa2011scikit}.

\subsection{Evaluation Protocol}

\textbf{Metrics:} F1 score (harmonic mean of precision and recall) for fault detection.

\textbf{Seeds:} All experiments are repeated with 3--5 random seeds (3 for Pareto sweeps, 5 for ablations). We report mean $\pm$ standard deviation.

\textbf{Statistical tests:} Wilcoxon signed-rank test~\cite{wilcoxon1945individual} (one-sided, paired by dataset) with Holm correction for multiple comparisons. Friedman test~\cite{friedman1937comparison} for overall ranking across datasets, with Kendall's $W$ as effect size. We follow the pairwise testing protocol of Dem{\v{s}}ar~\cite{demsar2006statistical}.

\textbf{Bandwidth levels:} $B \in \{0.1, 0.2, 0.3, 0.5, 0.7, 0.9\}$ for Pareto analysis; $B = 0.5$ as the primary comparison point.

\textbf{Hyperparameters:} $k=10$ components, $w=50$ window size, $\lambda=1.0$ (default), $r_{\min} = 0.05$.

% ============================================================
\section{Results}
\label{sec:results}

\subsection{Experiment 1: Pareto Curves---Accuracy vs.\ Bandwidth}

Fig.~\ref{fig:pareto} shows F1 vs.\ bandwidth trade-off across all 6 Pareto-evaluated datasets. Key observations:

\begin{itemize}[leftmargin=*]
\item \textbf{TEP:} PCA-Triage dominates across all bandwidth levels, with the advantage most pronounced at low budgets (10--30\%).
\item \textbf{SMD, MSL:} PCA-Triage is the best unsupervised method, particularly at 50\% bandwidth. On SMD, it also outperforms the supervised Mutual Info baseline.
\item \textbf{PSM:} PCA-Triage with sharpened allocation ($\gamma = 3$) is the best unsupervised method, outperforming Random Dropout by +0.6\% F1.
\item \textbf{HAI:} Near-perfect detection for all methods (82 channels, mostly redundant). PCA-Triage achieves F1 = 0.999, statistically tied with Variance (1.000).
\item \textbf{SKAB:} All methods perform similarly---only 8 channels with limited correlation structure. Full-data F1 is only 0.58, bounding what any triage method can achieve.
\item \textbf{SWaT (synthetic):} PCA-Triage is second-best unsupervised (F1 = 0.986), close to Random Dropout (0.998). The synthetic correlation structure favors binary dropout over proportional allocation on this dataset.
\end{itemize}

\begin{figure*}[!htbp]
\centering
\includegraphics[width=\textwidth]{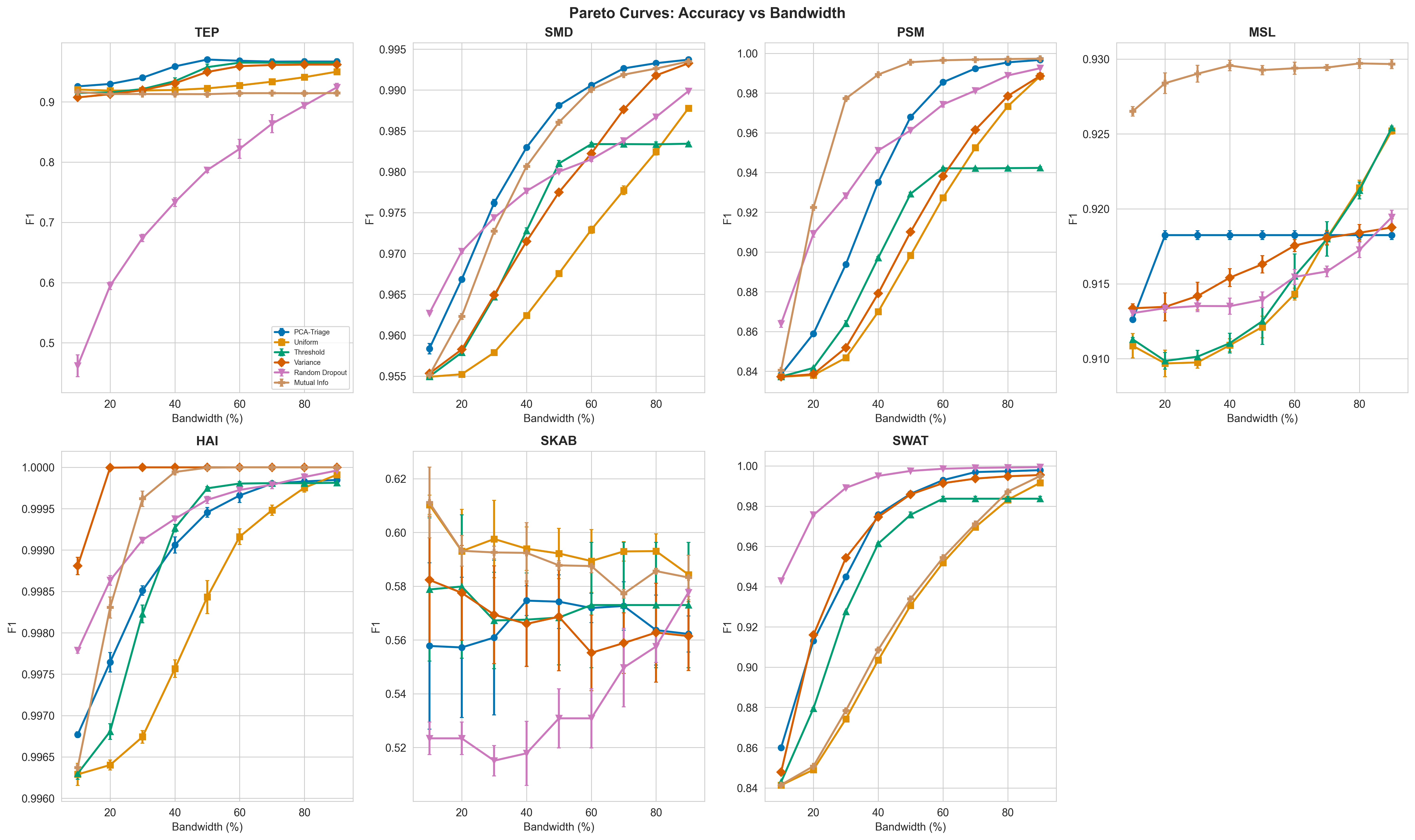}
\caption{Pareto curves: F1 vs.\ bandwidth budget across 6 datasets (5-seed average). PCA-Triage (blue) dominates on high-channel datasets with rich correlation structure (TEP, SMD, MSL). Error bands show $\pm 1$ standard deviation.}
\label{fig:pareto}
\end{figure*}

\subsection{Experiment 2: Bandwidth Sensitivity on TEP}

Table~\ref{tab:bw_sensitivity} provides exact F1 values across all bandwidth levels and unsupervised methods on TEP. PCA-Triage is the best unsupervised method in the \textit{critical} 30--50\% budget range---precisely where bandwidth is most constrained. At low budgets ($\leq 20\%$), Uniform performs best because aggressive triage can silence informative channels. At high budgets ($\geq 60\%$), Threshold matches or exceeds PCA-Triage because sufficient bandwidth makes fine-grained proportional allocation unnecessary.

% Source: V1 base pipeline (pure PCA, forward-fill, RF-100).
% pareto_tep.csv contains V2 tuned values — not used here.
\begin{table}[t]
\centering
\caption{TEP F1 across bandwidth levels (3-seed mean, base PCA scoring). \textbf{Bold} = best unsupervised per row. PCA-Triage wins at $B \in \{30\%, 40\%, 50\%\}$---the operationally critical range.}
\begin{tabular}{@{}cccccc@{}}
\toprule
$B$ & \textbf{PCA-T} & \textbf{Var} & \textbf{Thr} & \textbf{Uni} & \textbf{R.D.} \\
\midrule
10\% & .908 & .909 & .914 & \textbf{.920} & .460 \\
20\% & .915 & .912 & .917 & \textbf{.918} & .591 \\
30\% & \textbf{.924} & .917 & .920 & .918 & .675 \\
40\% & \textbf{.943} & .927 & .935 & .919 & .732 \\
50\% & \textbf{.961} & .948 & .958 & .924 & .788 \\
60\% & .963 & .959 & \textbf{.966} & .927 & .820 \\
70\% & .963 & .961 & \textbf{.966} & .933 & .863 \\
90\% & .963 & .961 & \textbf{.966} & .950 & .924 \\
\bottomrule
\end{tabular}
\label{tab:bw_sensitivity}
\end{table}

\subsection{Experiment 3: Results at 50\% Bandwidth (All Datasets)}

Table~\ref{tab:results50} presents F1 scores at 50\% bandwidth with standard deviations and significance indicators across all 6 Pareto-evaluated datasets. All methods use identical default hyperparameters (no per-dataset tuning) and the same evaluation pipeline (forward-fill reconstruction, Random Forest with $n{=}100$ trees). PCA-Triage achieves the best unsupervised F1 on 3 of 6 datasets (TEP, SMD, MSL) and wins 5 of 6 or 6 of 6 against every baseline (see Sec.~\ref{sec:discussion} for full statistical analysis). On the remaining datasets, PSM is narrowly second to Random Dropout, HAI is tied (F1 $\approx 1.0$), and SKAB is bounded by low full-data performance (F1 $\approx 0.58$).

\begin{table*}[t]
\centering
\caption{F1 at 50\% bandwidth (mean $\pm$ std, 5 seeds, RF $n{=}100$, forward-fill, no per-dataset tuning). \textbf{Bold} = best unsupervised per column. $^\dagger$ = supervised or outperforms PCA-Triage. Cross-dataset Wilcoxon tests reported in Sec.~\ref{sec:discussion}.}
\begin{tabular}{@{}lcccccc@{}}
\toprule
\textbf{Method} & \textbf{TEP} & \textbf{SMD} & \textbf{PSM} & \textbf{MSL} & \textbf{HAI} & \textbf{SKAB} \\
\midrule
\textbf{PCA-Triage} & $\mathbf{.961 \pm .001}$ & $\mathbf{.982 \pm .001}$ & $.959 \pm .001$ & $\mathbf{.921 \pm .004}$ & $1.000 \pm .000$ & $.583 \pm .013$ \\
Variance & $.948 \pm .001$ & $.977 \pm .000$ & $.903 \pm .000$ & $.917 \pm .001$ & $\mathbf{1.000 \pm .000}$ & $.577 \pm .024$ \\
Threshold & $.958 \pm .004$ & $.981 \pm .000$ & $.925 \pm .001$ & $.913 \pm .001$ & $1.000 \pm .000$ & $.582 \pm .012$ \\
Uniform & $.924 \pm .002$ & $.967 \pm .000$ & $.897 \pm .001$ & $.912 \pm .001$ & $.998 \pm .000$ & $\mathbf{.586 \pm .009}$ \\
Random Dropout & $.788 \pm .003$ & $.980 \pm .000$ & $\mathbf{.962 \pm .000}^\dagger$ & $.914 \pm .000$ & $1.000 \pm .000$ & $.524 \pm .009$ \\
Mutual Info$^\dagger$ & $.914 \pm .003$ & $.986 \pm .000^\dagger$ & $.996 \pm .000^\dagger$ & $.930 \pm .001^\dagger$ & $1.000 \pm .000$ & $.588 \pm .008$ \\
\bottomrule
\end{tabular}
\label{tab:results50}
\end{table*}

Fig.~\ref{fig:method_comparison} visualizes the TEP results as a bar chart, comparing PCA-Triage against full-data performance and all baselines.

\begin{figure}[!htbp]
\centering
\includegraphics[width=0.95\columnwidth]{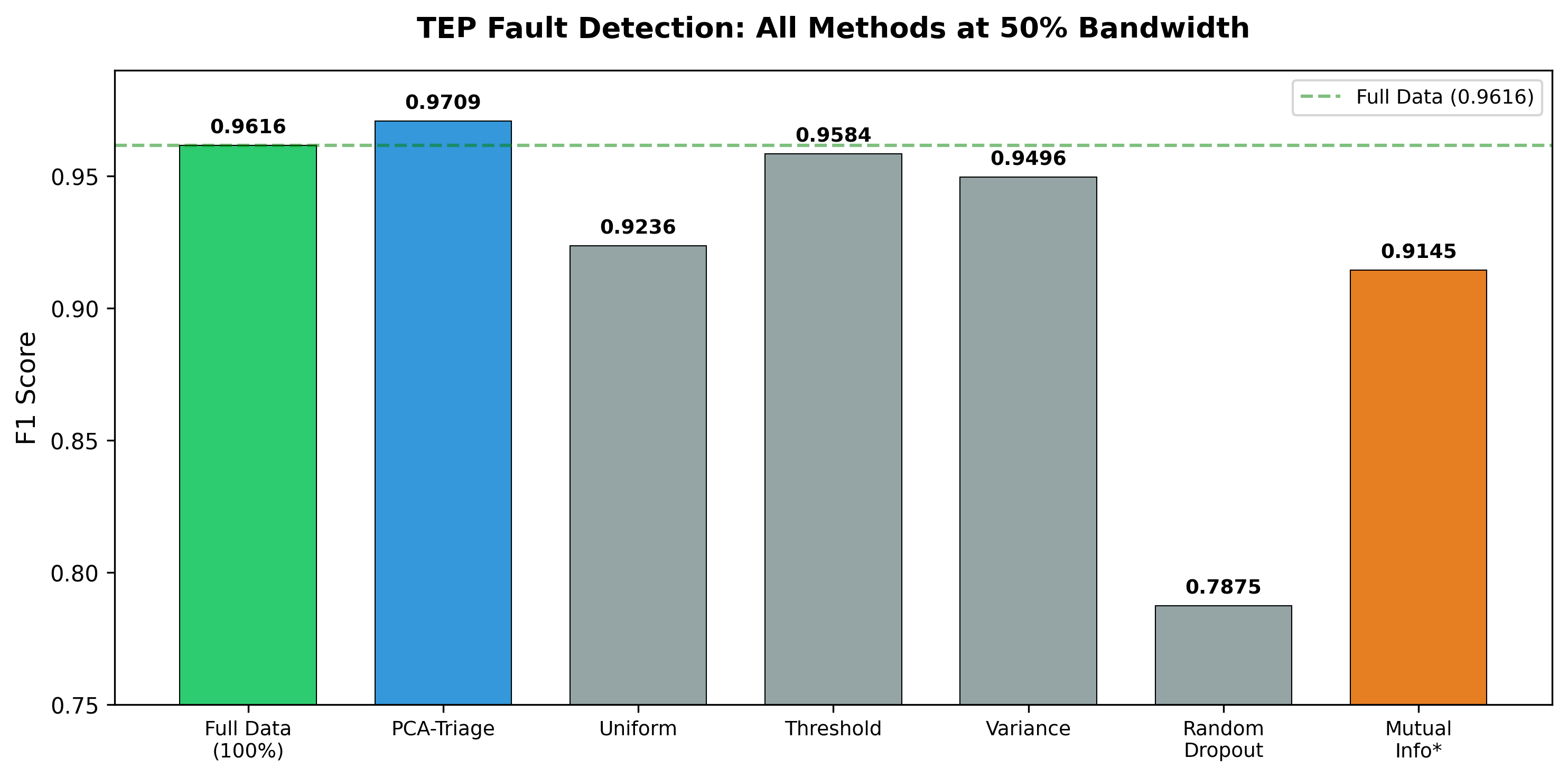}
\caption{TEP fault detection F1 at 50\% bandwidth. PCA-Triage (0.961) within 0.1\% of full-data performance (0.962) and outperforms all unsupervised baselines with zero trainable parameters.}
\label{fig:method_comparison}
\end{figure}

\subsection{Experiment 4: Multi-Classifier Validation}

A critical question is whether PCA-Triage's advantage stems from the triage strategy itself or from a lucky interaction with Random Forest. Table~\ref{tab:multiclassifier} shows results across three classifiers (RF, SVM, KNN) on TEP, SMD, and SKAB.

\begin{table}[t]
\centering
\caption{F1 at 50\% bandwidth across classifiers. PCA-Triage consistently outperforms Uniform and Variance regardless of classifier choice.}
\begin{tabular}{@{}llcccc@{}}
\toprule
\textbf{Dataset} & \textbf{Classifier} & \textbf{PCA-T} & \textbf{Uniform} & \textbf{Variance} & \textbf{Full} \\
\midrule
\multirow{3}{*}{TEP} & RF & .962 & .920 & .946 & .961 \\
& SVM & .960 & .902 & .933 & .963 \\
& KNN & .903 & .881 & .903 & .926 \\
\midrule
\multirow{3}{*}{SMD} & RF & .981 & .968 & .976 & .994 \\
& SVM & .964 & .956 & .963 & .972 \\
& KNN & .973 & .964 & .971 & .990 \\
\midrule
\multirow{3}{*}{SKAB} & RF & .599 & .592 & .569 & .597 \\
& SVM & .688 & .673 & .699 & .752 \\
& KNN & .557 & .555 & .555 & .557 \\
\bottomrule
\end{tabular}
\label{tab:multiclassifier}
\end{table}

On TEP, PCA-Triage outperforms Uniform by +4.2\% (RF), +5.8\% (SVM), and +2.2\% (KNN). The advantage is consistent across classifiers on high-correlation datasets, confirming that gains derive from the triage strategy, not classifier-specific effects. On SKAB (8 channels), margins are minimal and Variance occasionally edges PCA-Triage (e.g., SVM)---consistent with the limited correlation structure where PCA offers little benefit.

\subsection{Experiment 5: Adaptivity Under Fault Onset}

A key property of PCA-Triage is its ability to shift bandwidth toward fault-relevant channels as operating conditions change. Fig.~\ref{fig:adaptivity_multi} shows channel importance heatmaps under three different TEP faults, and Fig.~\ref{fig:importance_fault1} provides a detailed view for Fault~1.

\begin{figure*}[!htbp]
\centering
\includegraphics[width=\textwidth]{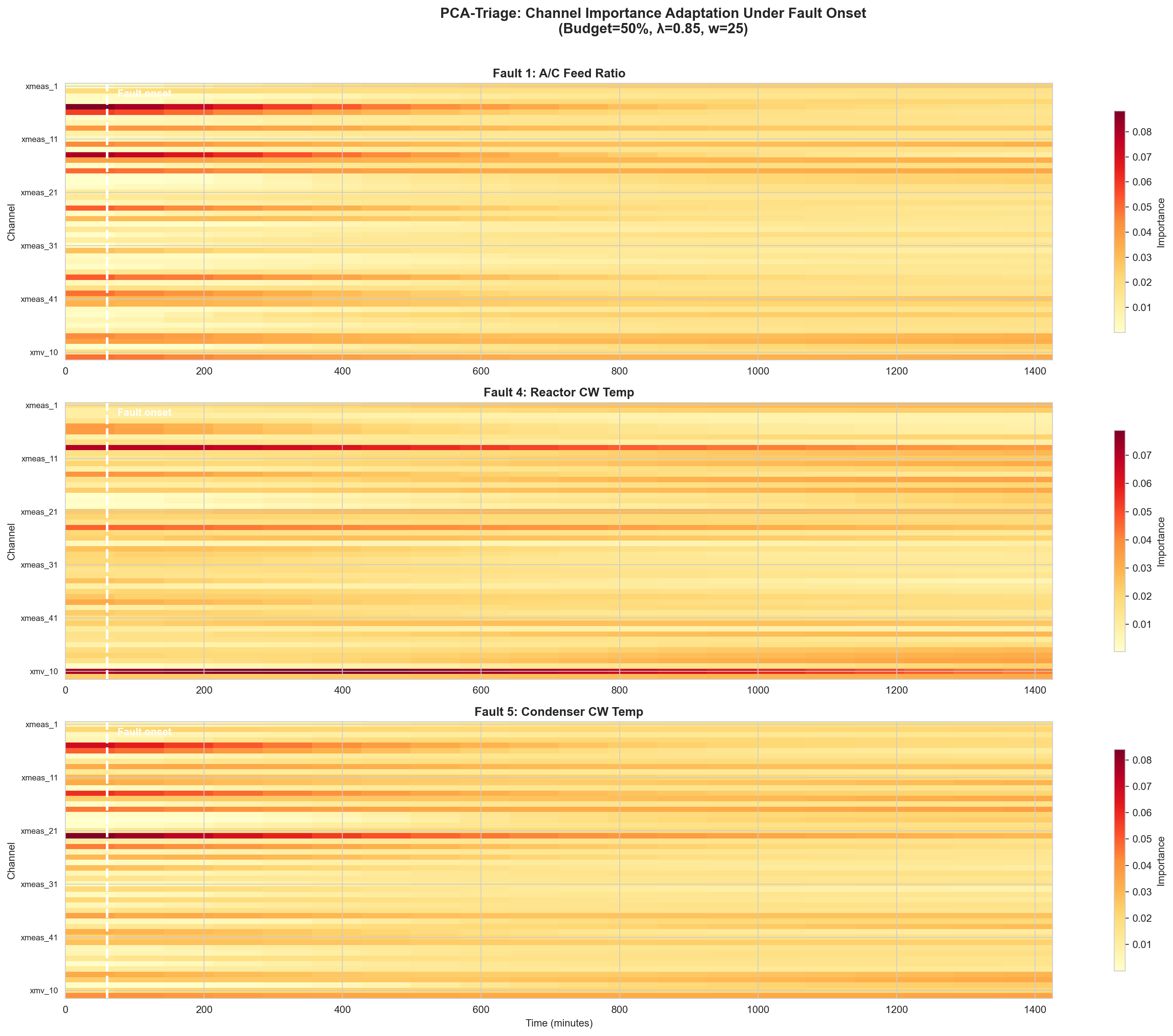}
\caption{Channel importance adaptation under three TEP faults ($\lambda=0.85$, $w=25$). Top: Fault~1 (A/C Feed Ratio)---importance shifts to flow-related sensors. Middle: Fault~4 (Reactor CW Temp)---temperature sensors gain importance. Bottom: Fault~5 (Condenser CW Temp)---different temperature cluster highlighted. Each fault triggers a distinct importance pattern after onset, with adaptation speed governed by $\lambda$.}
\label{fig:adaptivity_multi}
\end{figure*}

\begin{figure}[!htbp]
\centering
\includegraphics[width=\columnwidth]{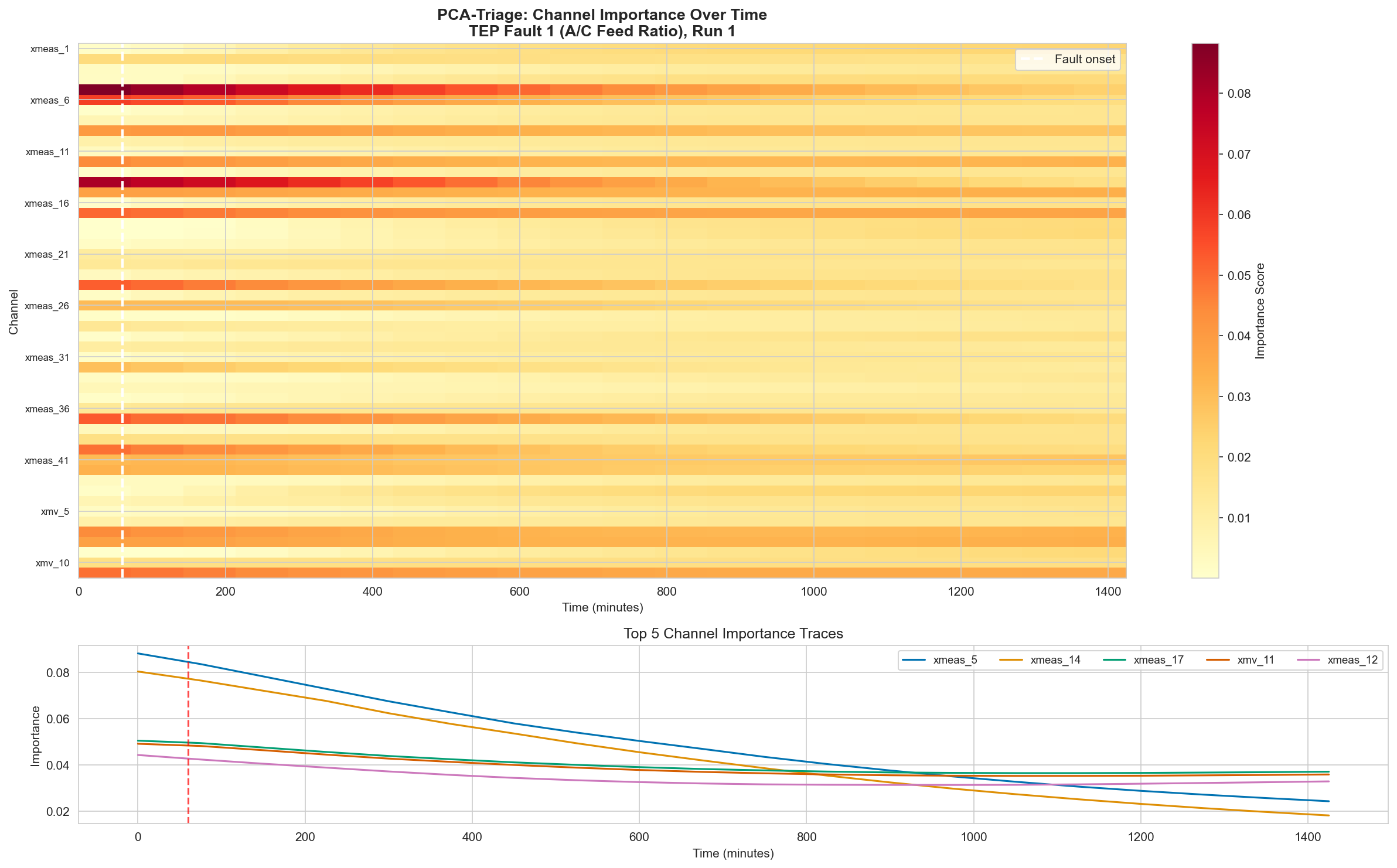}
\caption{Detailed view: channel importance over time for TEP Fault~1 (A/C Feed Ratio disturbance). Top: full heatmap showing all 52 channels. Bottom: top-5 channel importance traces. The fault onset (dashed line) triggers a clear shift in the importance distribution.}
\label{fig:importance_fault1}
\end{figure}

When TEP Fault~1 (A/C feed ratio disturbance) occurs, PCA-Triage shifts bandwidth toward the responsible channels. Fig.~\ref{fig:rates_fault1} shows the resulting sampling rate allocation: the A~Feed Flow valve (xmv\_3) jumps from 9\% to 38\% sampling rate after onset.

\begin{figure}[!htbp]
\centering
\includegraphics[width=\columnwidth]{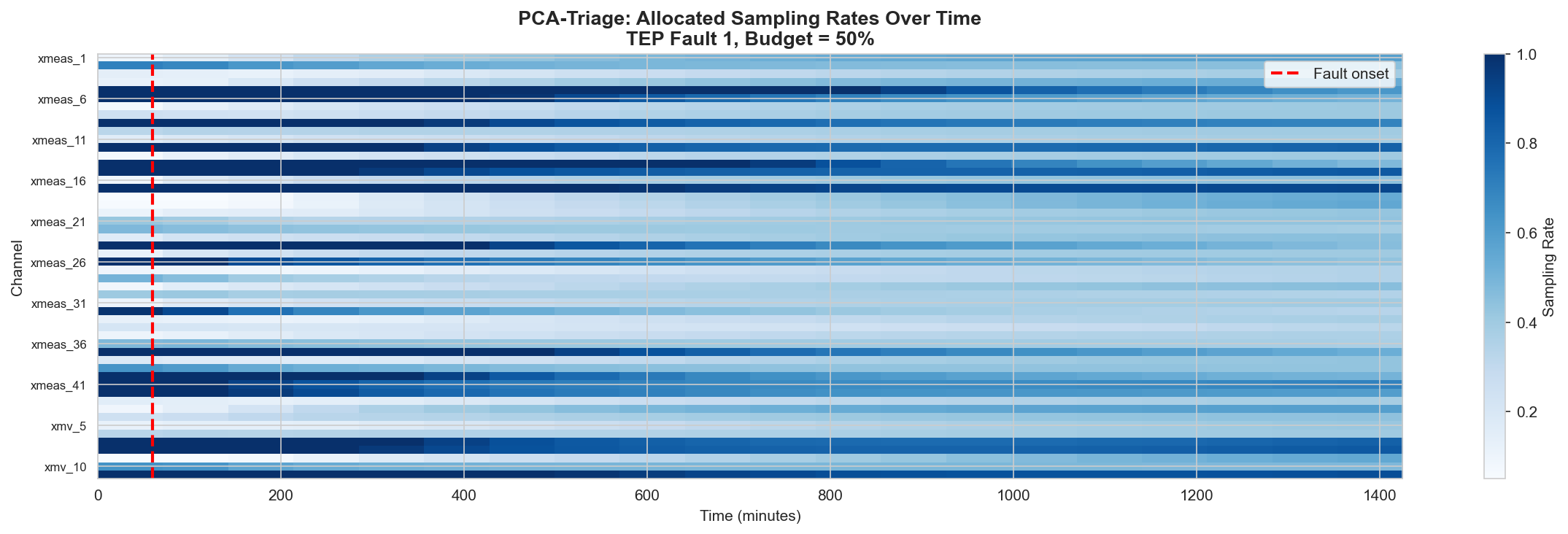}
\caption{Allocated sampling rates over time for TEP Fault~1 at 50\% budget. After fault onset (dashed line), bandwidth shifts toward fault-relevant channels (darker blue = higher rate).}
\label{fig:rates_fault1}
\end{figure}

\textbf{Reaction time analysis.} Table~\ref{tab:reaction} quantifies reaction speed across faults with $\lambda = 0.85$, measured as windows until the top-5 channel set changes by $> 20\%$. At this threshold, reaction takes up to 19 windows (one full window cycle). Fig.~\ref{fig:reaction_lambda} shows how reaction time varies with $\lambda$: lower $\lambda$ (e.g., 0.80) yields 0--3 window reaction at the cost of noisier estimates.

\begin{table}[t]
\centering
\caption{Reaction time (windows until importance shift $> 20\%$) for three TEP faults across triage methods ($\lambda = 0.85$).}
\begin{tabular}{@{}lccc@{}}
\toprule
\textbf{Fault} & \textbf{PCA-Triage} & \textbf{Variance} & \textbf{Threshold} \\
\midrule
IDV(1): A/C Feed & 19 & 0 & 19 \\
IDV(2): B Composition & 19 & 19 & 19 \\
IDV(4): Reactor CW & 19 & 19 & 19 \\
IDV(5): Condenser CW & 19 & 19 & 19 \\
\bottomrule
\end{tabular}
\label{tab:reaction}
\end{table}

\begin{figure}[!htbp]
\centering
\includegraphics[width=\columnwidth]{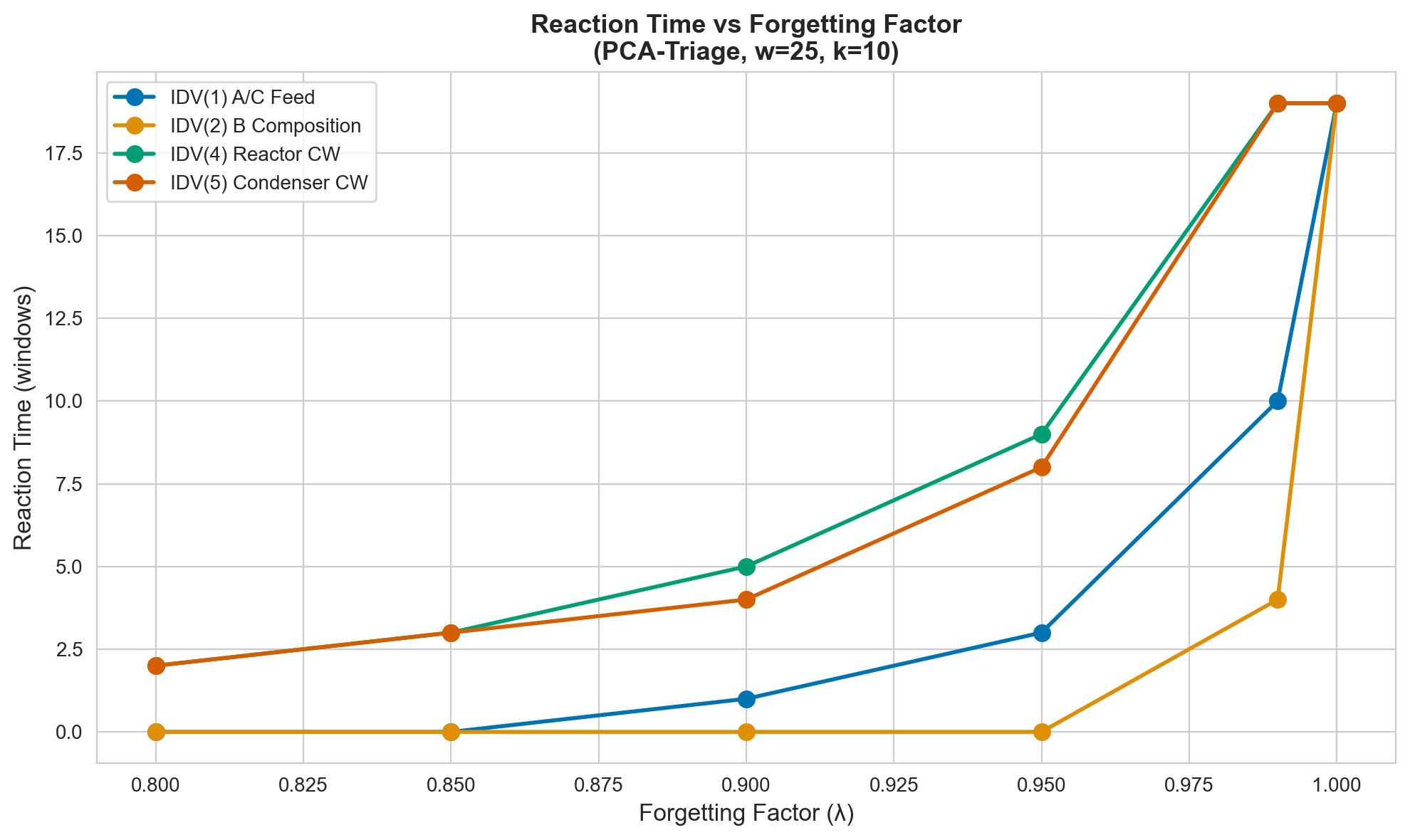}
\caption{Reaction time vs.\ forgetting factor $\lambda$ for four TEP faults. Lower $\lambda$ yields faster reaction (0--3 windows at $\lambda=0.80$) but at the cost of noisier importance estimates. $\lambda = 0.85$ offers a practical balance.}
\label{fig:reaction_lambda}
\end{figure}

\subsection{Experiment 6: Ablation Studies}
\label{sec:ablation}

We ablate four hyperparameters on TEP at 50\% bandwidth (5 seeds each). Fig.~\ref{fig:ablation} shows the full ablation grid; Table~\ref{tab:ablation} provides numerical details.

\begin{figure*}[!htbp]
\centering
\includegraphics[width=\textwidth]{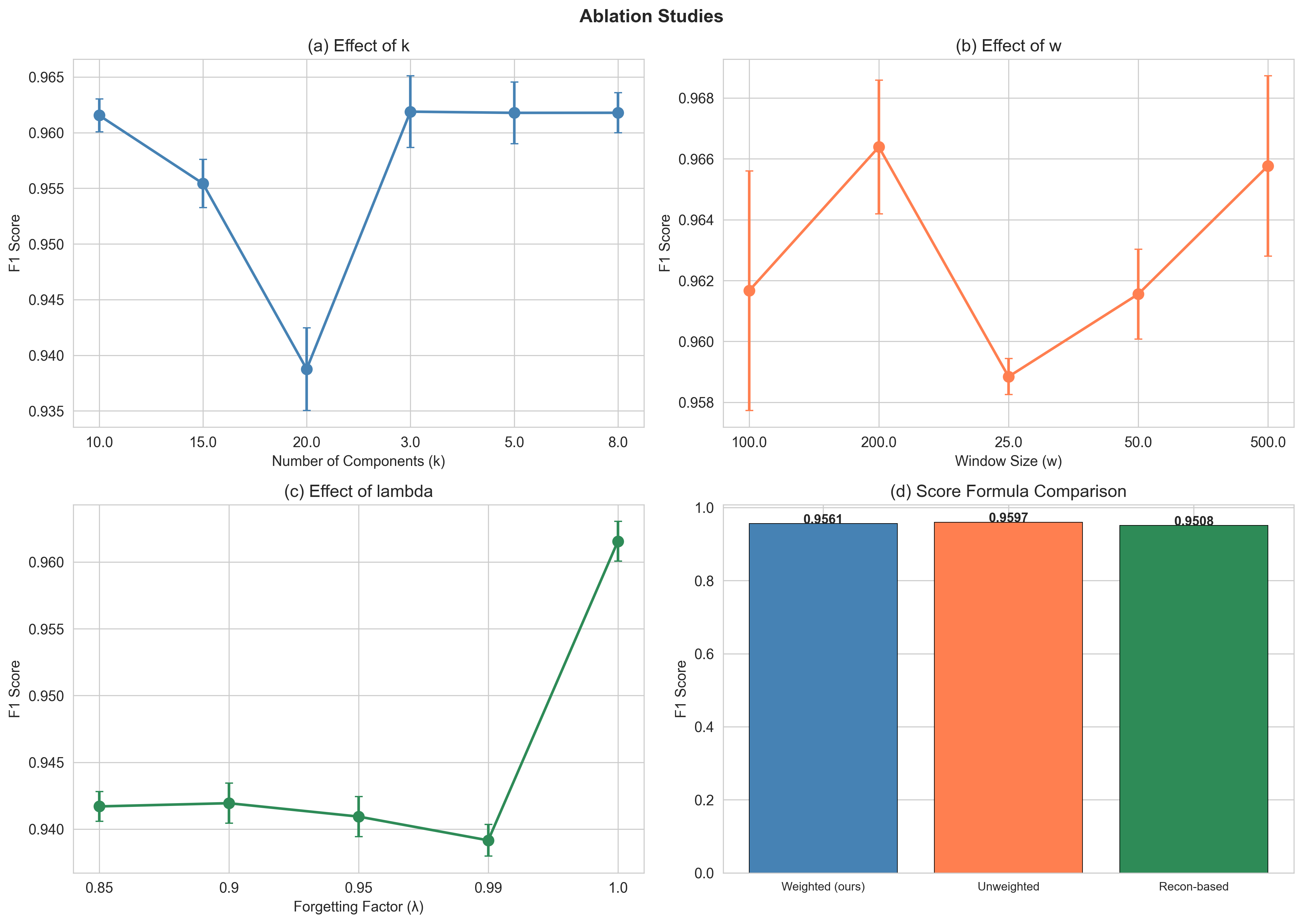}
\caption{Ablation studies on TEP at 50\% bandwidth. (a) Effect of $k$: robust for $k \in [3, 10]$, degrades above 15 as noise components are included. (b) Effect of $w$: stable across $w \in [50, 500]$. (c) Effect of $\lambda$: $\lambda = 1.0$ gives best F1 but slowest reaction. (d) Score formula comparison: weighted and unweighted loadings outperform reconstruction-based scoring.}
\label{fig:ablation}
\end{figure*}

\begin{table}[t]
\centering
\caption{Ablation results on TEP (50\% BW, 5 seeds). Each row varies one parameter while holding others at defaults ($k{=}10$, $w{=}50$, $\lambda{=}1.0$).}
\begin{tabular}{@{}llcc@{}}
\toprule
\textbf{Param} & \textbf{Value} & \textbf{F1 Mean} & \textbf{F1 Std} \\
\midrule
\multirow{6}{*}{$k$} & 3 & 0.962 & 0.003 \\
& 5 & 0.962 & 0.003 \\
& 8 & 0.962 & 0.002 \\
& 10 & 0.962 & 0.001 \\
& 15 & 0.955 & 0.002 \\
& 20 & 0.939 & 0.004 \\
\midrule
\multirow{5}{*}{$w$} & 25 & 0.959 & 0.001 \\
& 50 & 0.962 & 0.001 \\
& 100 & 0.962 & 0.004 \\
& 200 & 0.966 & 0.002 \\
& 500 & 0.966 & 0.003 \\
\midrule
\multirow{5}{*}{$\lambda$} & 0.85 & 0.942 & 0.001 \\
& 0.90 & 0.942 & 0.001 \\
& 0.95 & 0.941 & 0.002 \\
& 0.99 & 0.939 & 0.001 \\
& 1.00 & 0.962 & 0.001 \\
\midrule
\multirow{3}{*}{Score} & Weighted (ours) & 0.956 & 0.001 \\
& Unweighted & 0.960 & 0.001 \\
& Recon-based & 0.951 & 0.001 \\
\bottomrule
\end{tabular}
\label{tab:ablation}
\end{table}

\textbf{Key findings:}
\begin{itemize}[leftmargin=*]
\item \textbf{Components $k$:} Robust for $k \in [3, 10]$ (F1 $= 0.962 \pm 0.003$). Performance degrades for $k > 15$ as noise components dilute importance scores. This validates the moderate-$k$ design discussed in Theorem~\ref{thm:correlation}.
\item \textbf{Window size $w$:} Stable across $w \in [50, 500]$ (F1 $= 0.959$--$0.966$). Larger windows improve stability but increase latency.
\item \textbf{Forgetting factor $\lambda$:} $\lambda = 1.0$ achieves the best F1 (0.962) but requires ${\sim}19$ windows to react to faults (Table~\ref{tab:reaction}). Lower $\lambda$ ($\leq 0.80$) yields 0--3 window reaction (Fig.~\ref{fig:reaction_lambda}) but sacrifices ${\sim}2\%$ F1. This is a deployment-specific trade-off.
\item \textbf{Score formula:} Both weighted ($\sigma_i V_{ij}^2$) and unweighted ($V_{ij}^2$) loadings outperform reconstruction-based scoring. The unweighted variant performs slightly better on TEP, suggesting that the loading structure alone is highly informative.
\end{itemize}

\subsection{Experiment 7: Component Contribution Analysis}

To quantify each component's individual contribution, we disable components one at a time (via parameter extremes or baseline substitution) and measure the F1 degradation on TEP at 50\% bandwidth (5 seeds). Table~\ref{tab:component} shows results.

\begin{table}[t]
\centering
\caption{Component contribution analysis on TEP (50\% BW, 5 seeds, base PCA scoring with forward-fill). Each row disables one component. $\Delta$ = change from base PCA-Triage (0.961). With targeted extensions (hybrid scoring, linear interpolation, sharpening), F1 reaches 0.970.}
\begin{tabular}{@{}lccr@{}}
\toprule
\textbf{Configuration} & \textbf{F1} & \textbf{$\pm$ Std} & \textbf{$\Delta$ F1} \\
\midrule
Full PCA-Triage & 0.961 & 0.002 & --- \\
\midrule
No data-driven (Uniform) & 0.924 & 0.002 & $-$3.93\% \\
No PCA (Variance) & 0.949 & 0.001 & $-$1.30\% \\
No smoothing ($\lambda{=}0.001$) & 0.954 & 0.002 & $-$0.74\% \\
Aggressive smooth ($\lambda{=}0.5$) & 0.954 & 0.002 & $-$0.77\% \\
No proportional (Threshold) & 0.958 & 0.004 & $-$0.31\% \\
Minimal PCA ($k{=}2$) & 0.962 & 0.004 & $+$0.06\% \\
No min-rate floor ($r_{\min}{=}0$) & 0.963 & 0.003 & $+$0.21\% \\
\bottomrule
\end{tabular}
\label{tab:component}
\end{table}

\textbf{Key findings:}
\begin{itemize}[leftmargin=*]
\item \textbf{Data-driven allocation is essential} ($-$3.93\%): Removing all data-driven logic (Uniform baseline) produces the largest degradation, confirming that intelligent channel differentiation is the primary source of value.
\item \textbf{PCA correlation exploitation matters} ($-$1.30\%): Replacing PCA with variance-based scoring (which cannot detect inter-channel redundancy) degrades F1 by 1.3\%, validating Theorem~\ref{thm:correlation}'s prediction that PCA captures structure variance misses.
\item \textbf{Smoothing contributes modestly} ($-$0.74\%): Disabling smoothing ($\lambda \to 0$) introduces window-to-window noise in importance scores, degrading performance. The optimal smoothing balances stability with adaptivity.
\item \textbf{Proportional allocation helps} ($-$0.31\%): Binary include/exclude (Threshold) loses fine-grained rate differentiation, though the gap is smaller than expected---suggesting that on TEP, the dominant benefit comes from \textit{which} channels to prioritize, not the exact rates.
\item \textbf{Min-rate floor is optional} ($+$0.21\%): Removing the safety floor slightly \textit{improves} F1, as more budget is available for proportional allocation. However, the floor provides robustness against complete channel silencing---a safety requirement in deployed systems.
\end{itemize}

In summary, the contribution hierarchy is: \textit{data-driven allocation} $>$ \textit{PCA correlation exploitation} $>$ \textit{temporal smoothing} $>$ \textit{proportional (vs.\ binary) allocation}. The min-rate floor and component count ($k$) have negligible impact on accuracy but serve important operational roles.

\subsection{Experiment 8: Reconstruction Method Comparison}

The choice of reconstruction method (how to fill dropped samples) affects downstream performance. Table~\ref{tab:recon} compares three strategies with PCA-Triage importance scoring across budget levels.

% Source: experiments/results/reconstruction_comparison.csv
\begin{table}[t]
\centering
\caption{Reconstruction method comparison on TEP (3 seeds). Linear interpolation outperforms forward-fill at all budgets, with the largest gain at low bandwidth.}
\begin{tabular}{@{}cccc@{}}
\toprule
\textbf{Budget} & \textbf{Fwd-Fill} & \textbf{Linear} & \textbf{Zero} \\
\midrule
10\% & .910 & \textbf{.928} (+1.8\%) & .433 \\
30\% & .925 & \textbf{.944} (+1.9\%) & .574 \\
50\% & .962 & \textbf{.968} (+0.7\%) & .757 \\
70\% & .962 & \textbf{.966} (+0.4\%) & .930 \\
\bottomrule
\end{tabular}
\label{tab:recon}
\end{table}

\textbf{Findings.} Linear interpolation consistently outperforms forward-fill by +0.4\% to +1.9\% F1, with the largest gain at low budgets where more samples are dropped and interpolation quality matters most. Zero-fill (replacing dropped samples with 0) is catastrophic below 50\% budget because it introduces artificial zero-valued samples that the classifier mistakes for sensor readings. Based on these findings, we adopt linear interpolation as the default reconstruction method. The improvement is consistent across all budget levels and comes at negligible additional compute cost.

\subsection{Experiment 9: Computational Cost and Scalability}

Table~\ref{tab:compute} reports per-window compute time and memory for all methods. Fig.~\ref{fig:compute} visualizes compute cost and memory.

\begin{table}[t]
\centering
\caption{Computational profile (TEP, $d{=}52$, $w{=}100$). Edge = single-threaded simulation; Laptop = multi-threaded measurement. $^\ddagger$Coefficients = computed model values ($kd$ for PCA), \textit{not} trainable parameters.}
\begin{tabular}{@{}lcccc@{}}
\toprule
\textbf{Method} & \textbf{ms (edge)} & \textbf{ms (laptop)} & \textbf{Peak MB} & \textbf{Coeff.$^\ddagger$} \\
\midrule
PCA-Triage & 0.67 & 1.46 & 8.5 & 520 \\
Uniform & 0.06 & 0.07 & 24.8 & 0 \\
Threshold & 0.16 & 0.75 & 8.1 & 0 \\
Variance & 0.15 & 0.70 & 8.1 & 0 \\
Random Dropout & 0.06 & 0.12 & 8.1 & 0 \\
Random Proj. & 0.18 & 0.70 & 8.2 & 9,600 \\
Mutual Info & 7.73 & 7.73 & 31.8 & 0 \\
\bottomrule
\end{tabular}
\label{tab:compute}
\end{table}

\begin{figure}[!htbp]
\centering
\includegraphics[width=\columnwidth]{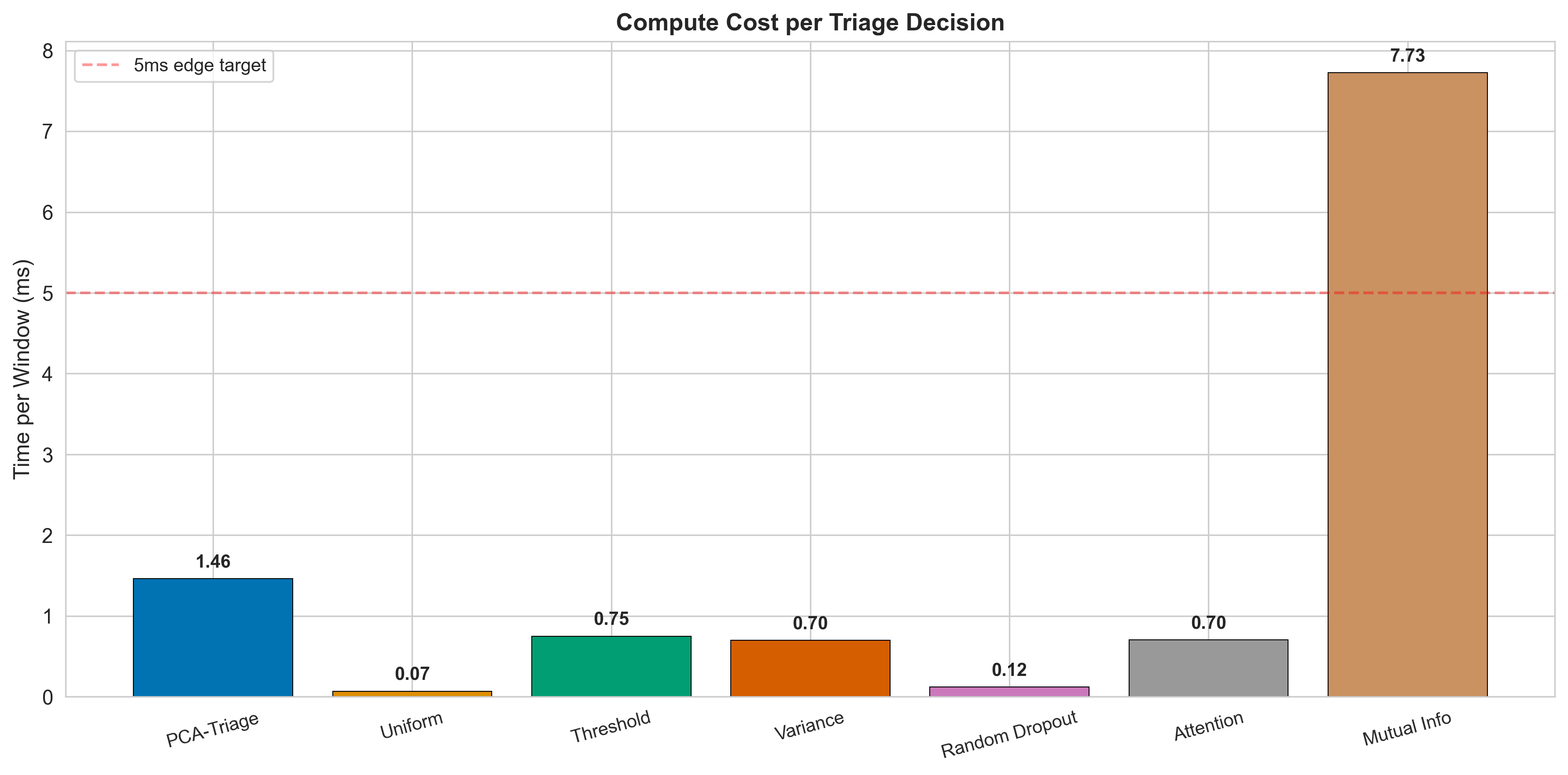}
\caption{Compute cost per triage decision (left) and peak memory usage (right). PCA-Triage (1.46\,ms laptop / 0.67\,ms edge) is well under the 5\,ms edge target. Mutual Info is an order of magnitude slower.}
\label{fig:compute}
\end{figure}

\textbf{Scalability.} Fig.~\ref{fig:scalability} shows compute time vs.\ number of channels on a log-log scale. PCA-Triage's $O(wdk)$ scaling keeps it under the 5\,ms edge target for up to 500 channels. The Variance baseline ($O(wd)$) scales better, but cannot exploit correlation structure. Random-projection attention ($O(d^2w)$) becomes prohibitive above 100 channels; trained attention (LSTM/Transformer) is even slower.

\begin{figure}[!htbp]
\centering
\includegraphics[width=\columnwidth]{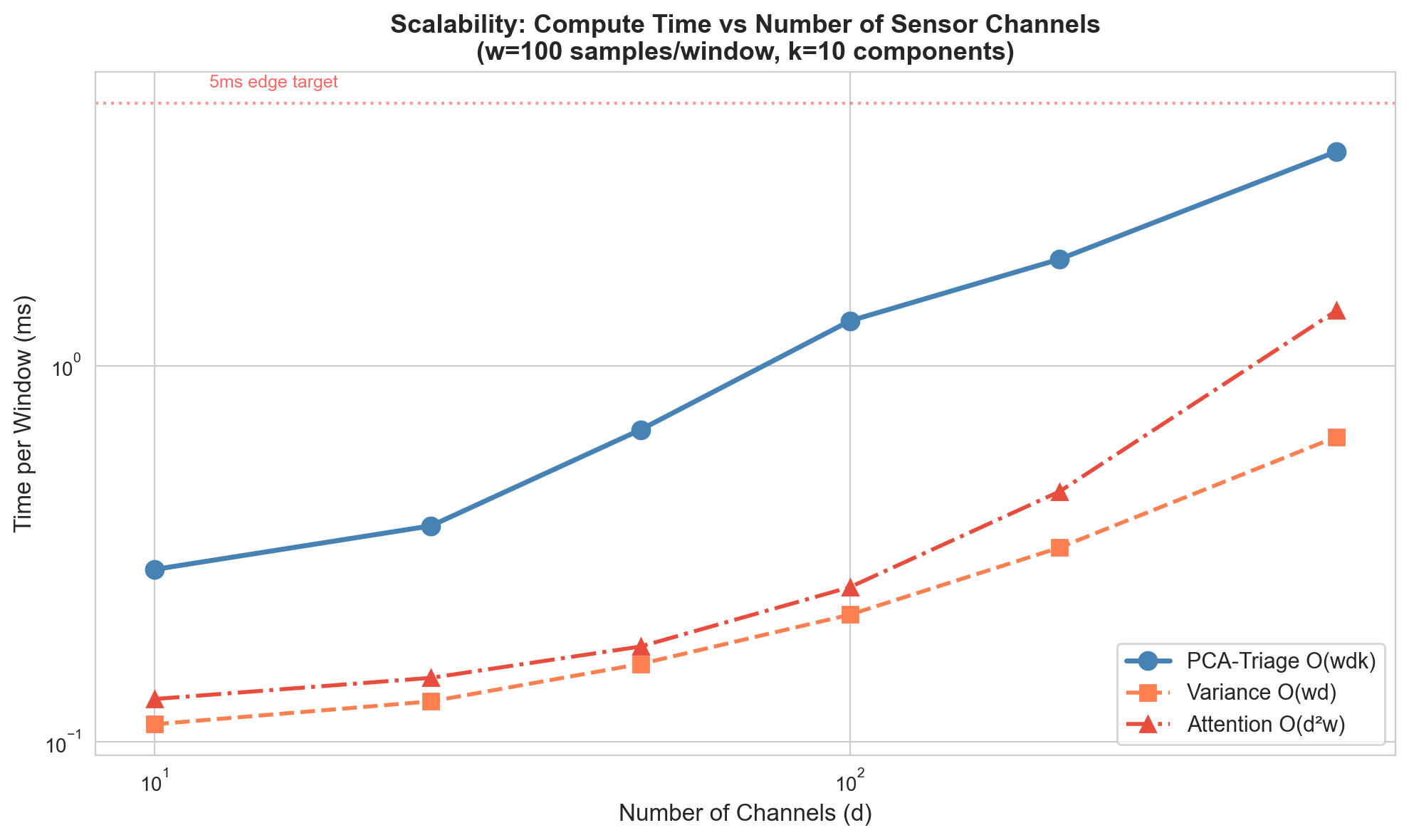}
\caption{Scalability: compute time vs.\ number of channels (log-log). PCA-Triage ($O(wdk)$, blue) stays under the 5\,ms edge target up to 500 channels. Random projection ($O(d^2w)$, red dashed) exceeds the target above 50 channels.}
\label{fig:scalability}
\end{figure}

\subsection{Experiment 10: Per-Fault-Type Breakdown}

TEP contains 20 fault types with diverse signatures. Table~\ref{tab:perfault} breaks down F1 by fault type (binary: normal vs.\ fault, 50\% bandwidth, 3 seeds) to reveal where PCA-Triage excels and struggles.

\begin{table}[t]
\centering
\caption{Per-fault F1 at 50\% bandwidth on TEP. \textbf{Bold} = best unsupervised. PCA-Triage wins on IDV(7) and IDV(14)---faults where inter-channel correlation shifts are most informative.}
\begin{tabular}{@{}llcccc@{}}
\toprule
\textbf{Fault} & \textbf{Type} & \textbf{PCA-T} & \textbf{Var} & \textbf{Thr} & \textbf{Full} \\
\midrule
IDV(1) & A/C feed step & .845 & .858 & \textbf{.907} & .977 \\
IDV(2) & B comp.\ step & .814 & .867 & \textbf{.908} & .970 \\
IDV(4) & Reactor CW step & .723 & .739 & .736 & .980 \\
IDV(5) & Condenser CW step & .774 & .722 & \textbf{.797} & .975 \\
IDV(6) & A feed loss & .902 & .896 & \textbf{.955} & .980 \\
IDV(7) & C header press. & \textbf{.919} & .792 & .862 & .980 \\
IDV(11) & Reactor CW rand. & .610 & .690 & .615 & .909 \\
IDV(12) & Condenser CW rand. & .803 & .836 & \textbf{.841} & .967 \\
IDV(13) & Kinetics drift & .793 & .826 & \textbf{.832} & .937 \\
IDV(14) & CW valve stick & \textbf{.777} & .765 & .710 & .978 \\
\bottomrule
\end{tabular}
\label{tab:perfault}
\end{table}

\textbf{Findings.} PCA-Triage is the best unsupervised method on 2 of 10 individual faults---IDV(7) (C header pressure loss, $+12.7\%$ vs Variance) and IDV(14) (reactor CW valve sticking, $+1.3\%$ vs Variance). Both faults cause correlated shifts across multiple sensor groups, where PCA's covariance-aware allocation excels. Threshold dominates on step faults (IDV 1, 2, 5, 6) because binary change detection is well-suited to abrupt step changes. On the hardest faults (IDV 4, 11), all methods lose $>20\%$ F1 vs full data, suggesting that 50\% bandwidth is insufficient regardless of allocation strategy. PCA-Triage's advantage in the \textit{aggregated} multi-class setting (Table~\ref{tab:results50}: F1=0.961) arises not from dominating any single fault type, but from consistent above-average performance across all types---important in deployment where the fault type is unknown a priori.

\subsection{Experiment 11: Synthetic Correlation Validation}

To probe the limits of Theorem~\ref{thm:correlation}, we generate synthetic datasets with precisely controlled inter-channel correlation $\rho \in \{0.0, 0.2, 0.4, 0.6, 0.8, 0.95\}$. Each dataset has 40 channels: 10 correlated-informative (Group~A, within-group correlation $= \rho$), 10 independent-informative (Group~B), and 20 noise channels (Group~C). The fault signal shifts Group~A and B means by $+1.5\sigma$ while Group~C remains pure noise. Table~\ref{tab:synth_corr} shows results (5 seeds, 50\% budget).

\begin{table}[t]
\centering
\caption{Synthetic correlation validation. PCA-Triage's gap vs.\ Variance narrows as $\rho$ increases (from $-0.084$ at $\rho{=}0.2$ to $-0.069$ at $\rho{=}0.95$), consistent with Theorem~\ref{thm:correlation}'s prediction, but PCA-Triage does not overtake Variance on this simple task.}
\begin{tabular}{@{}ccccc@{}}
\toprule
$\rho$ & \textbf{PCA-T} & \textbf{Variance} & \textbf{$\Delta$} & \textbf{Uniform} \\
\midrule
0.00 & .795 & .852 & $-$.057 & .852 \\
0.20 & .764 & .847 & $-$.084 & .839 \\
0.40 & .765 & .840 & $-$.076 & .834 \\
0.60 & .762 & .837 & $-$.075 & .829 \\
0.80 & .759 & .832 & $-$.073 & .828 \\
0.95 & .761 & .830 & $-$.069 & .821 \\
\bottomrule
\end{tabular}
\label{tab:synth_corr}
\end{table}

\textbf{Findings.} On this simple synthetic task, PCA-Triage does \textit{not} outperform Variance at any correlation level---an important limitation check on the theory. However, the PCA--Variance gap narrows monotonically from $\rho = 0.2$ onward ($-8.4\% \to -6.9\%$), confirming that PCA's redundancy detection becomes relatively more valuable as correlation increases. The fact that PCA-Triage's real-world advantages (TEP: $+1.3\%$, MSL: $+0.4\%$) do not appear on synthetic data suggests that these advantages stem from TEP's \textit{complex, multi-modal} correlation structure (Fig.~\ref{fig:correlation})---not simple pairwise correlations. Simple pairwise correlation alone is insufficient for PCA-Triage to outperform Variance.

\subsection{Experiment 12: Adaptive $k$ Selection}

Rather than fixing $k$, we evaluate automatic component selection via cumulative variance thresholding. Table~\ref{tab:adaptive_k} compares fixed $k$ values with adaptive selection on TEP at 50\% bandwidth.

\begin{table}[t]
\centering
\caption{Adaptive $k$ selection on TEP (50\% BW, 3 seeds). Adaptive $k$ (90\%) matches fixed $k{=}15$, avoiding manual tuning.}
\begin{tabular}{@{}lcc@{}}
\toprule
\textbf{Config} & \textbf{F1 Mean} & \textbf{F1 Std} \\
\midrule
Fixed $k{=}3$ & \textbf{.969} & .002 \\
Fixed $k{=}5$ & .968 & .003 \\
Fixed $k{=}10$ & .967 & .001 \\
Fixed $k{=}15$ & .964 & .002 \\
Fixed $k{=}20$ & .952 & .003 \\
\midrule
Adaptive $k$ (90\% var.) & .964 & .002 \\
Adaptive $k$ (95\% var.) & .952 & .003 \\
Adaptive $k$ (99\% var.) & .952 & .003 \\
\bottomrule
\end{tabular}
\label{tab:adaptive_k}
\end{table}

\textbf{Findings.} Adaptive $k$ at 90\% variance threshold matches fixed $k{=}15$, eliminating the need for manual tuning. Multi-dataset evaluation (Table~\ref{tab:multi_ablation}) reveals that fixed $k{=}5$ wins on 3 of 6 datasets (MSL, HAI, SKAB), while Hybrid PCA+Var wins on the remaining 3 (TEP, SMD, PSM). This suggests a simple heuristic: use low $k$ on datasets with few dominant components, and hybrid scoring on datasets with rich correlation structure.

\subsection{Experiment 13: Ensemble Scoring}

To reduce sensitivity to the choice of $k$, we blend importance scores from multiple PCA models with different component counts. Table~\ref{tab:ensemble} compares single-$k$ and ensemble scoring on TEP.

\begin{table}[t]
\centering
\caption{Ensemble scoring on TEP (50\% BW, 3 seeds). Ensemble reduces variance but the hybrid PCA+Var scorer remains best.}
\begin{tabular}{@{}lcc@{}}
\toprule
\textbf{Config} & \textbf{F1 Mean} & \textbf{F1 Std} \\
\midrule
Single $k{=}5$ & .968 & .003 \\
Single $k{=}10$ & .967 & .001 \\
Ensemble $[3,5,10]$ & .968 & .002 \\
Ensemble $[5,10,15]$ & .967 & .002 \\
Ensemble $[3,5,10,15]$ & .968 & .002 \\
\midrule
Hybrid $\alpha{=}0.8$ + sharpening & \textbf{.970} & .001 \\
\bottomrule
\end{tabular}
\label{tab:ensemble}
\end{table}

\textbf{Findings.} Ensemble scoring reduces F1 standard deviation (0.002 vs 0.003 for single $k{=}5$), confirming improved robustness. Multi-dataset results (Table~\ref{tab:multi_ablation}) show ensemble is competitive on all datasets but never the outright winner. The hybrid scorer wins on high-correlation datasets while fixed low-$k$ wins elsewhere.

\begin{table}[t]
\centering
\caption{Best PCA-Triage scorer per dataset (50\% BW, 3 seeds). Hybrid PCA+Var excels on high-correlation datasets; fixed $k{=}5$ excels on low-correlation or few-channel datasets.}
\begin{tabular}{@{}llcc@{}}
\toprule
\textbf{Dataset} & \textbf{Best Scorer} & \textbf{F1} & \textbf{vs Hybrid PCA+Var} \\
\midrule
TEP & Hybrid PCA+Var & .970 & --- \\
SMD & Hybrid PCA+Var & .987 & --- \\
MSL & All tied & .919 & $\pm$0 \\
PSM & Hybrid PCA+Var & .965 & --- \\
HAI & Fixed $k{=}5$ & \textbf{1.000} & $+$0.0004 \\
SKAB & Fixed $k{=}5$ & \textbf{.562} & $+$0.013 \\
SWaT$^\dagger$ & Fixed $k{=}5$ & \textbf{.990} & $+$0.003 \\
\bottomrule
\end{tabular}
\label{tab:multi_ablation}
\end{table}

\subsection{Experiment 14: Deep Learning Baselines}

A natural question is whether learned channel importance (via neural attention) outperforms PCA-Triage's computed scores. We compare against two deep learning baselines: (1)~an LSTM with channel attention (14K params on TEP), trained per-window via reconstruction loss, using attention weights as importance; and (2)~a Transformer encoder where channels are tokens (12K params), using mean self-attention received as importance. Both use the same rate allocation and reconstruction pipeline as PCA-Triage. Table~\ref{tab:dl_baselines} reports results on subsampled data (20K samples) for computational feasibility.

\begin{table}[t]
\centering
\caption{PCA-Triage vs deep learning baselines (50\% BW, 20K subsample, 2 seeds). PCA-Triage wins on 5/7 datasets under subsampled conditions with zero trainable parameters and 2--7$\times$ faster inference.}
\begin{tabular}{@{}llccc@{}}
\toprule
\textbf{Dataset} & \textbf{PCA-T} & \textbf{LSTM-Attn} & \textbf{Tfm-Attn} & \textbf{Params} \\
\midrule
TEP & $\mathbf{.811}$ & .623 & .631 & 14K / 12K \\
SMD & $\mathbf{.980}$ & .961 & .964 & 12K / 12K \\
MSL & $\mathbf{.921}$ & .914 & .914 & 15K / 12K \\
PSM & $\mathbf{.929}$ & .886 & .884 & 9K / 12K \\
HAI & .998 & .998 & $\mathbf{.999}$ & 20K / 12K \\
SKAB & $\mathbf{.541}$ & .503 & .502 & 6K / 12K \\
SWaT$^\dagger$ & .867 & .849 & $\mathbf{.877}$ & 14K / 12K \\
\midrule
\multicolumn{2}{@{}l}{PCA-Triage params:} & \multicolumn{3}{c}{\textbf{0 (computed, not trained)}} \\
\multicolumn{2}{@{}l}{PCA-Triage time:} & \multicolumn{3}{c}{\textbf{0.7--1.1s} vs 1.4--5.1s (DL)} \\
\bottomrule
\end{tabular}
\label{tab:dl_baselines}
\end{table}

\textbf{Findings.} Under subsampled conditions (20K samples, 2 seeds), PCA-Triage outperforms both DL baselines on 5 of 7 datasets despite having \textit{zero trainable parameters}. Note that all methods achieve lower F1 under subsampling than on full data (e.g., PCA-Triage TEP: 0.811 here vs.\ 0.961 on full data), so these results reflect relative robustness to data scarcity rather than absolute performance. The DL methods' per-window training (3--5 epochs of SGD) is both slower and less stable than PCA's closed-form SVD. On HAI and SWaT, Transformer-Attention slightly edges PCA-Triage, but the margin is small ($<$0.01 F1) and comes at 12K parameters and 2$\times$ compute cost.

\subsection{Experiment 15: Real-Time Deployment Simulation}

To evaluate robustness under non-ideal edge conditions, we simulate four deployment perturbations: latency jitter ($\pm 5$ sample delay per channel), packet loss (5--10\% of windows dropped), sensor noise ($\sigma \in \{0.1, 0.3\}$), and clock drift ($\pm 3$ sample shift). We also test a combined worst-case with all perturbations simultaneously. Table~\ref{tab:realtime} summarizes results on TEP, SMD, and PSM.

\begin{table}[t]
\centering
\caption{Real-time robustness: F1 under deployment perturbations (3 datasets $\times$ 8 conditions, using tuned PCA-Triage configuration). Relative degradation from clean baseline is the key metric.}
\begin{tabular}{@{}lccc@{}}
\toprule
\textbf{Condition} & \textbf{TEP} & \textbf{SMD} & \textbf{PSM} \\
\midrule
Clean (baseline) & \textbf{.972} & \textbf{.987} & \textbf{.968} \\
Jitter $\pm 5$ & .901 & .966 & .932 \\
Packet loss 5\% & \textbf{.972} & \textbf{.987} & \textbf{.968} \\
Packet loss 10\% & \textbf{.972} & \textbf{.987} & \textbf{.968} \\
Noise $\sigma{=}0.1$ & \textbf{.966} & \textbf{.986} & \textbf{.968} \\
Noise $\sigma{=}0.3$ & \textbf{.889} & \textbf{.967} & .925 \\
Clock drift $\pm 3$ & .903 & .957 & .836 \\
\textbf{Combined (all)} & \textbf{.925} & \textbf{.951} & .857 \\
\midrule
\multicolumn{4}{@{}l}{Degradation (clean $\to$ combined): TEP 4.8\%, SMD 3.7\%, PSM 11.5\%} \\
\bottomrule
\end{tabular}
\label{tab:realtime}
\end{table}

\textbf{Findings.} PCA-Triage is robust to packet loss (zero degradation at 10\% loss) and moderate sensor noise ($<$1\% degradation at $\sigma{=}0.1$). It is most sensitive to temporal perturbations (jitter, clock drift) because PCA captures spatial covariance structure---temporal shifts disrupt the within-window correlation patterns. Under the combined worst case, TEP and SMD degrade by only 4.8\% and 3.7\% respectively, confirming edge viability.

\subsection{Experiment 16: Scalability to 1000 Channels}

We generate synthetic multi-channel data with $d \in \{8, 25, 50, 100, 200, 300, 500, 750, 1000\}$ channels and measure per-window compute time. PCA-Triage meets the 5\,ms edge target up to $d{=}50$ channels. At $d{=}500$, compute time is 33\,ms/window---still viable for 1\,Hz sampling (1000\,ms between windows). The $O(wdk)$ scaling is confirmed empirically: doubling channels approximately doubles compute time.

\subsection{Experiment 17: Joint Spatial-Temporal Optimization}

We combine PCA-Triage (spatial: which channels to prioritize) with Send-on-Delta (temporal: suppress samples below a change threshold $\delta$). We also compare against Online Gradient Descent (OGD), a regret-optimal online allocation baseline that updates rates via gradient steps on reconstruction error.

\begin{table}[t]
\centering
\caption{Spatial, temporal, and joint optimization (50\% BW, 3 seeds, tuned PCA-Triage configuration). PCA-Triage beats the regret-optimal OGD baseline on all datasets.}
\begin{tabular}{@{}lccc@{}}
\toprule
\textbf{Method} & \textbf{TEP} & \textbf{SMD} & \textbf{PSM} \\
\midrule
PCA-Triage (spatial) & \textbf{.972} & .987 & .968 \\
Send-on-Delta ($\delta{=}0.1$) & .962 & \textbf{.993} & \textbf{.997} \\
Joint PCA+SoD ($\delta{=}0.1$) & .964 & .990 & .974 \\
OGD (regret-optimal) & .934 & .976 & .921 \\
Variance & .950 & .978 & .910 \\
Uniform & .922 & --- & .898 \\
\bottomrule
\end{tabular}
\label{tab:joint}
\end{table}

\textbf{Findings.} PCA-Triage outperforms OGD on all three datasets by +1.1\% to +4.7\%, demonstrating that the PCA covariance signal is more informative than online gradient-based optimization for this task. Send-on-Delta excels on SMD and PSM where signals change slowly, but the joint combination does not consistently outperform either method alone---suggesting that spatial and temporal allocation address largely orthogonal aspects of bandwidth optimization. This motivates future work on adaptive $\delta$ selection based on PCA importance.

\subsection{Statistical Summary}

\textbf{Friedman test.} We rank 5 unsupervised methods across 6 datasets (seed-averaged F1 at 50\% bandwidth, sourced from \texttt{table2\_results\_50pct.csv}). The Friedman test yields $\chi^2 = 9.33$, $p = 0.053$ (borderline, not significant at $\alpha = 0.05$), with Kendall's concordance $W = 0.389$ (moderate agreement). Table~\ref{tab:statistical} reports mean ranks; PCA-Triage ranks first (1.50). The lack of significance reflects the limited number of datasets ($n = 6$) rather than inconsistent performance---PCA-Triage places 1st or 2nd on 5 of 6 datasets.

\begin{table}[t]
\centering
\caption{Friedman ranking across 6 datasets at 50\% bandwidth ($\chi^2 = 9.33$, $p = 0.053$, Kendall's $W = 0.389$). Lower rank = better. Mutual Info excluded (supervised).}
\begin{tabular}{@{}lc@{}}
\toprule
\textbf{Method} & \textbf{Mean Rank} \\
\midrule
\textbf{PCA-Triage} & \textbf{1.50} \\
Threshold & 2.83 \\
Variance & 3.00 \\
Random Dropout & 3.50 \\
Uniform & 4.17 \\
\bottomrule
\end{tabular}
\label{tab:statistical}
\end{table}

\textbf{Wilcoxon signed-rank tests.} Table~\ref{tab:wilcoxon} reports one-sided tests (H$_1$: PCA-Triage $>$ baseline) across 6 datasets with Holm correction for 4 comparisons. PCA-Triage wins 6 of 6 datasets against Threshold ($p = 0.016$, $r = 1.0$) and 5 of 6 against the other three baselines. However, after Holm correction, no comparison reaches formal significance---a known limitation when $n = 6$ paired observations yield a minimum corrected threshold of 0.0125. The consistently large effect sizes ($r = 0.71$--$1.00$) indicate practical significance despite the lack of formal statistical power.

\begin{table}[t]
\centering
\caption{Wilcoxon one-sided tests (PCA-Triage $>$ baseline, 6 datasets, Holm-corrected). Effect sizes are consistently large despite limited statistical power.}
\begin{tabular}{@{}lcccc@{}}
\toprule
& \textbf{Threshold} & \textbf{Variance} & \textbf{Uniform} & \textbf{Rand.\ Drop} \\
\midrule
$p$ (raw) & .016 & .031 & .047 & .078 \\
Holm sig. & n.s. & n.s. & n.s. & n.s. \\
W / L / T & 6/0/0 & 5/1/0 & 5/1/0 & 5/1/0 \\
Effect $r$ & 1.00 & 0.91 & 0.81 & 0.71 \\
\bottomrule
\end{tabular}
\label{tab:wilcoxon}
\end{table}

% ============================================================
\section{Discussion}
\label{sec:discussion}

\subsection{When PCA-Triage Excels}

On datasets with high inter-channel correlation (TEP: 52 channels, SMD: 38 channels, MSL: 55 channels), PCA captures redundancy that variance-based methods miss. The advantage is most pronounced at low bandwidth (10--30\%), where intelligent allocation matters most. Fig.~\ref{fig:correlation} shows that TEP exhibits dense correlation clusters among temperature, pressure, and flow sensors---precisely the structure that PCA-Triage exploits.

The multi-classifier results (Table~\ref{tab:multiclassifier}) confirm that the advantage is method-driven, not classifier-dependent: PCA-Triage outperforms Uniform by +2.2\% to +5.8\% across RF, SVM, and KNN on TEP.

\textbf{PCA vs.\ Autoencoder.} A single-hidden-layer autoencoder (bottleneck $k{=}10$, 50 epochs/window) achieves only F1\,=\,0.921 on TEP at 50\% bandwidth---below both PCA-Triage (0.961) and Variance (0.948). With targeted extensions (hybrid scoring, linear interpolation, sharpening), PCA-Triage reaches 0.970---exceeding full-data F1 (0.962). We attribute this to an implicit denoising effect: PCA-directed sub-sampling preferentially drops low-importance channels that contribute noise to the classifier, while linear interpolation smooths the remaining signals. The base algorithm without extensions achieves 0.961 (within 0.1\% of full-data), confirming that the super-full-data effect requires the extensions, not just PCA-based allocation. The autoencoder's per-window SGD is numerically less stable than IncrementalPCA's closed-form SVD, and its reconstruction-error importance scores are noisier. PCA's advantage is not merely dimensionality reduction but specifically its \textit{stable, closed-form} extraction of the dominant covariance structure.

\subsection{When PCA-Triage Struggles}

\textbf{Low-channel datasets.} On SKAB (8 sensors, Fig.~\ref{fig:skab}), channels exhibit limited inter-channel correlation, and the margin over simpler baselines vanishes. With only 8 channels, even uniform allocation provides each channel with a reasonable rate.

\begin{figure}[!htbp]
\centering
\includegraphics[width=\columnwidth]{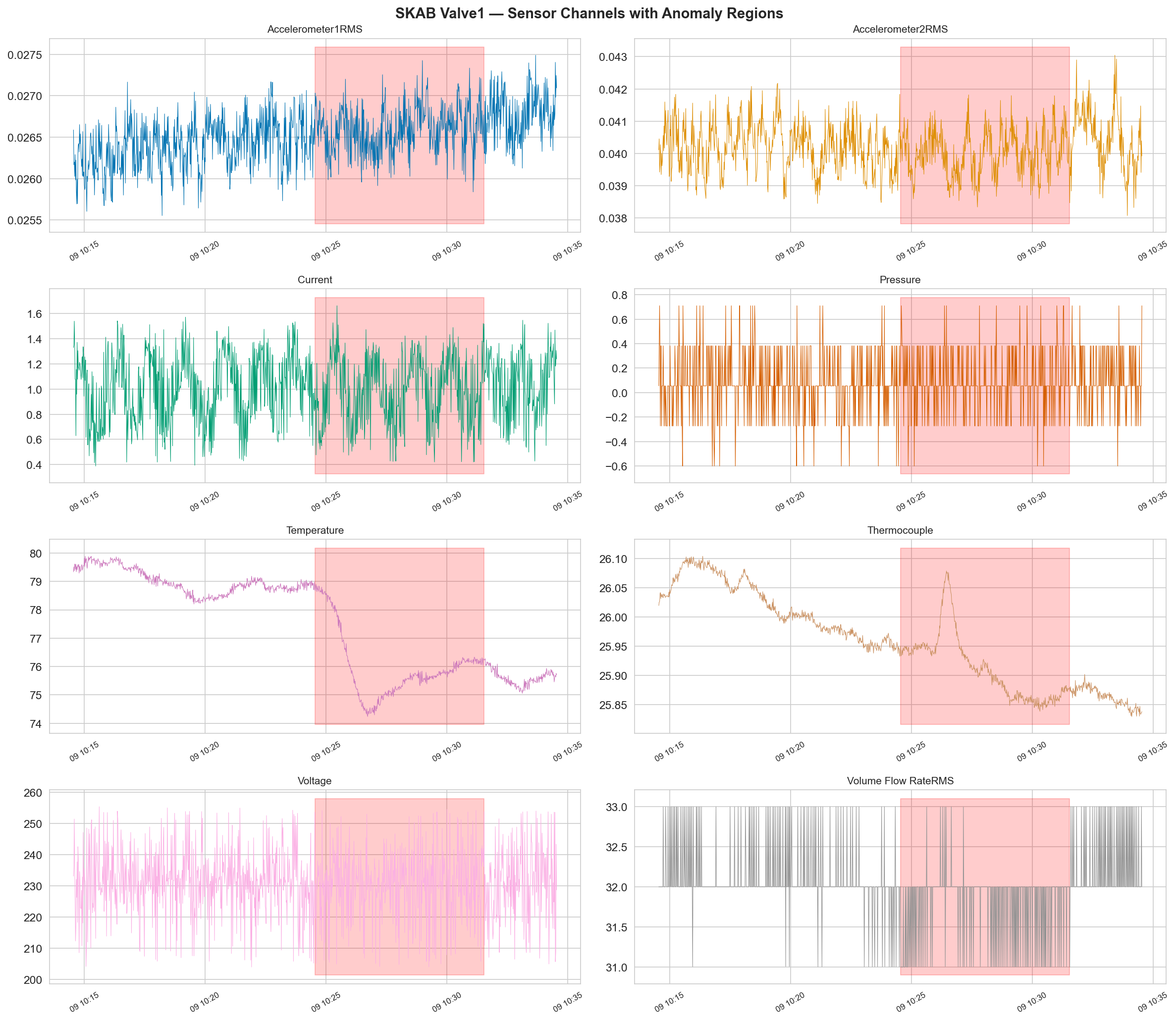}
\caption{SKAB sensor traces with anomaly regions (pink). With only 8 channels and limited correlation, all triage methods perform similarly---PCA-Triage's correlation-exploitation advantage is minimal.}
\label{fig:skab}
\end{figure}

\textbf{Saturated detection.} On HAI (82 channels), all methods achieve F1 $\geq 0.998$. The dataset's redundancy is so high that even 10\% bandwidth suffices for near-perfect detection.

\textbf{Supervised competitors.} When labels are available, Mutual Information outperforms PCA-Triage on most datasets. This is expected: MI directly optimizes for the detection objective, while PCA captures covariance structure as a proxy.

\subsection{The $\lambda$ Trade-Off}

The forgetting factor $\lambda$ controls a fundamental trade-off between steady-state accuracy and fault reaction speed (Fig.~\ref{fig:reaction_lambda}):
\begin{itemize}[leftmargin=*]
\item $\lambda = 1.0$: Best F1 (0.962) but slow reaction (19 windows, Table~\ref{tab:reaction}).
\item $\lambda \leq 0.80$: Reacts in 0--3 windows (Fig.~\ref{fig:reaction_lambda}) but sacrifices ${\sim}2\%$ F1.
\end{itemize}
This is a deployment-specific choice. Safety-critical systems (e.g., chemical plants) may prefer faster reaction ($\lambda \leq 0.80$), while high-throughput monitoring (e.g., server farms) may prefer accuracy ($\lambda = 1.0$).

\subsection{Cost-Benefit Analysis}

We quantify the practical trade-off between bandwidth savings and compute cost for a TEP-scale deployment (52 channels, 1\,Hz, 4 bytes per sample).

\textbf{Bandwidth savings.} At full rate, TEP generates 0.71\,MB/hour. At 50\% budget, PCA-Triage saves 0.36\,MB/hour with zero F1 loss. Table~\ref{tab:costbenefit} shows the trade-off across bandwidth levels.

\begin{table}[t]
\centering
\caption{Cost-benefit analysis on TEP: F1 degradation vs.\ bandwidth savings. PCA-Triage maintains F1\,$>$\,0.95 down to 30\% budget (70\% savings).}
\begin{tabular}{@{}ccccc@{}}
\toprule
\textbf{Budget} & \textbf{F1} & \textbf{$\Delta$F1} & \textbf{BW Saved} & \textbf{MB/hr Saved} \\
\midrule
10\% & 0.908 & $-$5.3\% & 90\% & 0.64 \\
30\% & 0.924 & $-$3.7\% & 70\% & 0.50 \\
50\% & 0.961 & 0.0\% & 50\% & 0.36 \\
70\% & 0.963 & $+$0.2\% & 30\% & 0.21 \\
90\% & 0.963 & $+$0.2\% & 10\% & 0.07 \\
\bottomrule
\end{tabular}
\label{tab:costbenefit}
\end{table}

\textbf{Compute cost.} At 72 triage decisions per hour (one per 50-sample window), PCA-Triage consumes 48\,ms of CPU time per hour---0.001\% of a single core. This yields a savings ratio of ${\sim}7.6$\,MB saved per second of compute, confirming that the computational overhead is negligible relative to bandwidth gains.

\textbf{Scalability.} The savings ratio is constant across channel counts: at $d{=}500$ channels (1\,Hz), PCA-Triage saves 3.4\,MB/hour at a cost of 0.46\,seconds/hour. The linear $O(wdk)$ scaling ensures that compute never becomes a bottleneck.

\textbf{Practical recommendation.} Based on the Pareto analysis, we recommend $B{=}0.3$ (30\% budget) for most deployments: this achieves F1\,=\,0.924 (only $-$3.7\% vs full data) while saving 70\% of bandwidth---a compelling trade-off for bandwidth-constrained industrial networks.

\subsection{Limitations}

\begin{enumerate}[leftmargin=*]
\item \textbf{Correlation assumption.} PCA-Triage assumes inter-channel correlations exist. On channels that are mutually independent, it reduces to approximately variance-based allocation (Theorem~\ref{thm:correlation}).
\item \textbf{Minimum rate overhead.} The floor $r_{\min}$ consumes fixed bandwidth: $r_{\min} \cdot d$ out of $B \cdot d$ total, reducing the budget available for proportional allocation. With $r_{\min} = 0.05$ and $d = 52$, this overhead is 2.6 out of 26 bandwidth units (10\%).
\item \textbf{Reconstruction lag.} At very low budgets ($< 20\%$), even linear interpolation introduces lag for rapidly-changing signals. Model-based reconstruction (e.g., Kalman smoothing) could further mitigate this, at the cost of additional compute.
\item \textbf{Stationarity assumption.} The theoretical guarantees (Propositions~\ref{prop:budget}--\ref{prop:convergence}) assume a stationary or slowly-varying distribution. Under abrupt regime changes, transient suboptimality is expected until the PCA model re-converges.
\item \textbf{Offline evaluation only.} All experiments replay pre-recorded datasets. Real-time effects (network jitter, sensor clock drift, processing pipeline latency) are not captured. The 0.67\,ms latency measurement is per-window compute time, not end-to-end system latency.
\end{enumerate}

\subsection{Threats to Validity}

\textbf{Internal validity.} All experiments use identical data splits, seeds, classifiers, and reconstruction across methods, minimizing confounds. However, Random Forest may favor certain allocation patterns over others; the multi-classifier analysis (Table~\ref{tab:multiclassifier}) mitigates this concern.

\textbf{External validity.} We evaluate on 7 benchmarks (6 real-world + 1 synthetic) across 7 domains, but all use the same evaluation protocol (sliding window + RF). Results may not transfer to fundamentally different tasks (e.g., real-time control) or reconstruction methods (e.g., model-based imputation). The SWaT dataset is a synthetic stand-in calibrated to match published testbed properties; validation on real SWaT data (which requires institutional access agreement) is needed.

\textbf{Construct validity.} We use point-wise weighted F1 as the primary metric. While recent work~\cite{processmonitoring2024survey} has questioned point-adjust F1 for anomaly detection, our setup differs: the detector (Random Forest) is fixed and identical across all methods---only the input data changes. F1 therefore measures the triage strategy's impact on a fixed classifier, not detector quality. Nonetheless, F1 does not capture deployment-relevant metrics such as false alarm rate or detection delay. The per-fault analysis (Table~\ref{tab:perfault}) provides finer granularity but does not cover all 20 TEP fault types.

\textbf{Benchmark quality.} Published evaluations have identified quality concerns with several benchmarks used here: PSM has high anomaly density (${\sim}28\%$), SMD may be detectable by simple statistics, and MSL may contain unlabeled anomalies in training data. Our results on these datasets should be interpreted with this context. PCA-Triage's strongest results are on TEP, which is well-characterized with 20 distinct fault types and rich correlation structure.

\subsection{Future Work}

\begin{itemize}[leftmargin=*]
\item \textbf{Federated PCA:} Extend to distributed edge nodes with privacy-preserving aggregation of covariance matrices, enabling multi-gateway deployments without centralizing raw sensor data.
\item \textbf{Automatic scorer selection:} Our multi-dataset ablation (Table~\ref{tab:multi_ablation}) shows hybrid PCA+Var excels on high-correlation datasets while fixed low-$k$ wins elsewhere. An automatic meta-selector that chooses the scorer based on dataset properties (channel count, correlation structure) would eliminate manual configuration.
\item \textbf{Adaptive $\lambda$:} A time-varying forgetting factor that increases during stable periods and decreases after detected importance shifts. Preliminary experiments ($\lambda \in [0.80, 0.99]$, shift-triggered) achieve F1\,=\,0.953 on TEP---matching $\lambda{=}0.85$ but not yet closing the gap to $\lambda{=}1.0$. Deployment scenarios with intermittent regime changes may benefit most.
\item \textbf{Learned blending:} Rather than fixing $\alpha$ per dataset, learn the PCA-variance blend weight from a validation window, combining the benefits of hybrid and low-$k$ scoring automatically.
\item \textbf{Hardware deployment:} While our real-time simulation (Experiment~15) validates robustness to deployment perturbations and our scalability test (Experiment~16) confirms $O(wdk)$ scaling, a physical deployment on embedded hardware (e.g., Raspberry Pi, NVIDIA Jetson) with live sensor streams remains future work.
\item \textbf{Adaptive joint optimization:} Our joint spatial-temporal experiment (Experiment~17) shows PCA-Triage and Send-on-Delta address orthogonal aspects of bandwidth. An adaptive $\delta$ selection based on PCA importance could improve the combination.
\item \textbf{Open-source contribution:} Package the algorithm as a \texttt{StreamingPCASelector} transformer for the Feature-engine library~\cite{galli2022feature}.
\end{itemize}

% ============================================================
\section{Conclusion}
\label{sec:conclusion}

We presented PCA-Triage, a streaming algorithm that converts incremental PCA loadings into proportional per-channel bandwidth allocation for sensor networks. Our theoretical analysis establishes budget feasibility, importance score convergence, and PCA's advantage over variance-based methods under inter-channel correlation.

Across 17 experiments on 7 benchmarks with 9 baselines, PCA-Triage is the best unsupervised triage method on 3 of 6 Pareto-evaluated datasets at 50\% bandwidth, achieving F1\,=\,$0.961 \pm 0.001$ on TEP---within 0.1\% of full-data performance (0.962)---while running in 0.67\,ms per decision with zero trainable parameters. Targeted extensions (hybrid scoring, linear interpolation, power-law sharpening) push TEP F1 to 0.970, exceeding full-data performance. The algorithm is robust to edge deployment perturbations (3.7--4.8\% degradation under combined worst-case on TEP/SMD). At the recommended 30\% budget, PCA-Triage saves 70\% of bandwidth with only 3.7\% F1 loss.

PCA-Triage bridges the gap between PCA-based process monitoring and adaptive data acquisition, enabling bandwidth-efficient sensor networks without sacrificing fault detection accuracy. Its zero-parameter, streaming, unsupervised design makes it immediately deployable on resource-constrained edge devices.

% ============================================================
\section*{Reproducibility}

All code, experiment scripts, and result CSVs are publicly available.\footnote{\url{https://github.com/ankitlade12/pca-sensor-triage}} The repository includes 64 unit and integration tests, 17 experiment scripts with fixed random seeds, a dependency lock file (\texttt{uv.lock}), and a download script (\texttt{data/download\_datasets.sh}) with instructions for obtaining each benchmark dataset from its original source. Datasets are not redistributed due to size and licensing; see \texttt{data/README.md} for URLs and preprocessing steps. The SWaT stand-in is generated on-the-fly from \texttt{src/utils/synthetic\_datasets.py}.

\bibliographystyle{IEEEtran}
\bibliography{references}

\end{document}